\documentclass[a4paper,11pt,oneside]{report}
\usepackage[MScThesis]{EPFLreport}
\usepackage{xspace}
\usepackage{todonotes}
\usepackage{mathtools}
\usepackage{caption}
\usepackage{thmtools, thm-restate}
\usepackage{layouts}
\usepackage{subcaption}
\usepackage{float}
\usepackage{amsmath,amsfonts,bm}

\def\eqref#1{equation~\ref{#1}}

\def\1{\bm{1}}

\def\vzero{{\bm{0}}}

\def\vmu{{\bm{\mu}}}
\def\vsigma{{\bm{\sigma}}}
\def\vtheta{{\bm{\theta}}}
\def\veta{{\bm{\eta}}}
\def\vphi{{\bm{\phi}}}
\def\va{{\bm{a}}}
\def\vb{{\bm{b}}}

\def\vf{{\bm{f}}}
\def\vg{{\bm{g}}}

\def\vk{{\bm{k}}}

\def\vm{{\bm{m}}}

\def\vp{{\bm{p}}}

\def\vr{{\bm{r}}}

\def\vv{{\bm{v}}}

\def\vx{{\bm{x}}}
\def\vy{{\bm{y}}}

\def\mA{{\bm{A}}}

\def\mF{{\bm{F}}}

\def\mI{{\bm{I}}}
\def\mJ{{\bm{J}}}
\def\mK{{\bm{K}}}

\def\mM{{\bm{M}}}

\def\mS{{\bm{S}}}

\def\mW{{\bm{W}}}

\def\mY{{\bm{Y}}}

\def\mLambda{{\bm{\Lambda}}}
\def\mSigma{{\bm{\Sigma}}}

\DeclareMathAlphabet{\mathsfit}{\encodingdefault}{\sfdefault}{m}{sl}
\SetMathAlphabet{\mathsfit}{bold}{\encodingdefault}{\sfdefault}{bx}{n}
\newcommand{\tens}[1]{\bm{\mathsfit{#1}}}

\def\tH{{\tens{H}}}

\def\tT{{\tens{T}}}

\def\gP{{\mathcal{P}}}

\newcommand{\R}{\mathbb{R}}

\newcommand{\softmax}{\mathrm{softmax}}
\newcommand{\sigmoid}{\sigma}

\newcommand{\KL}{D_{\mathrm{KL}}}

\newcommand{\Cov}{\mathrm{Cov}}
\newcommand{\diag}{\mathrm{diag}}

\DeclareMathOperator*{\argmax}{arg\,max}
\DeclareMathOperator*{\argmin}{arg\,min}

\newcommand{\rnd}[1]{\left(#1\right)}
\newcommand{\sqr}[1]{\left[#1\right]}
\newcommand{\crl}[1]{\left\{#1\right\}}

\newcommand{\myexpect}{\mathbb{E}}
\newcommand{\myvar}{\mathbb{V}}
\newcommand{\mycov}{\textrm{Cov}}

\newcommand{\logits}{f}
\newcommand{\vlogits}{\vf}
\newcommand{\linlogits}{\vf_{\textrm{lin}}}
\newcommand{\linlink}{\vg^{-1}_{\textrm{lin}}}

\newcommand{\fgp}{\vf_{\textrm{GP}}}
\newcommand{\GP}{\mathcal{GP}}
\newcommand{\param}{\vtheta}
\newcommand{\data}{\mathcal{D}}

\newcommand{\given}{\mbox{$|$}}

\newtheorem{theorem}{Theorem}[chapter]

\newcommand{\gauss}{\mathcal{N}}
\newcommand{\maparam}{\param_\textrm{MAP}}

\newcommand{\vkappa}{\bm{\kappa}}

\newcommand{\elbo}{\mathcal{L}}
\newcommand{\natparam}{\veta^{(1)}}
\newcommand{\natparas}{\veta^{(2)}}
\newcommand{\exparam}{\vphi^{(1)}}
\newcommand{\exparas}{\vphi^{(2)}}

\numberwithin{equation}{chapter}

\title{Disentangling the Gauss-Newton Method and Approximate Inference for Neural Networks}
\author{Alexander Immer}
\supervisor{Dr.~Emtiyaz~Khan\textsuperscript{$*$}}
\adviser{Prof.~Dr.~Matthias Grossglauser\textsuperscript{$\dagger$}\\
Prof.~Dr.~Patrick Thiran\textsuperscript{$\dagger$}}

\dedication{
\begin{raggedleft}
    You can't connect the dots looking forward;\\
    You can only connect them looking backwards.\\
    --- Steve Jobs\\
\end{raggedleft}
}
\acknowledgments{
I conducted the work presented in this thesis during my internship at RIKEN AIP in Tokyo in the approximate Bayesian inference team.
I am thankful that Emtiyaz Khan, on the day of my arrival in Tokyo, overwhelmed but also trusted me with his idea to combine approximate inference for neural networks with Gaussian process inference.
In three intensive months, Ehsan, Maciej, Emti and I finished the paper ``Approximate Inference Turns Deep Networks into Gaussian Processes''.I also want to thank Emti for giving me the freedom to pursue my own ideas after our paper, which ultimately led to this thesis that builds on it.
This work would not have been possible without all the great discussions with other interns and members of the team as well as their useful feedback.
In particular, I would like to thank Dharmesh Tailor, Ehsan Abedi, Maciej Korzepa, Matthias Bauer, Pierre Alquier, Pierre Oreistein, Pingbo Pan, Roman Bachmann,  Runa Eschenhagen, and Xiangming Meng.
I would also like to thank Olivier Dousse who inspired me to do research and Victor Kristof, Matthias Grossglauser, and Patrick Thiran for their supervision and support at EPFL.

I am especially thankful to my friends and family who made my time at EPFL and RIKEN so enjoyable.
I would like to thank my friends Danijar Hafner, Fabian Windheuser, Thomas Kellermeier, Tilman Kemeny, and Willi Raschkowski for visiting me in Tokyo.
I would also like to thank my friends at EPFL, Christian Sciuto, Federico Pucci, Lorenzo Tarantino, Martin Josifoski, and Mazen Al-ahdal for the amazing time in Lausanne.
Thank you, Arian Ghaemi, Ayaka Kitani, and Runa Eschenhagen for the incredible time in Tokyo.
Thank you, Ayaka, for your love, support, and encouragement.
I would also like to thank my friends in Berlin, who make me feel at home every time I return or visit.
In particular Jan Westphal, who I have not mentioned yet.
Last but not least, I would like to thank my family for their love and support in every decision I make.

}

\begin{document}
\maketitle

\begin{abstract}
Deep neural networks achieve state-of-the-art performance in many real-world machine learning problems and alleviate the need to design features by hand.
However, their flexibility often comes at a cost.
Neural network models are hard to interpret, often overconfident, and do not quantify how probable they are given a dataset.
The Bayesian approach to infer neural networks is one way to tackle these issues.
However, exact Bayesian inference for neural networks is intractable.
Therefore, \emph{Bayesian deep learning} combines approximate inference and optimization methods to design efficient methods that provide an approximate solution.
Nonetheless, the combination of both methods is often not well understood.

In this thesis, we disentangle the generalized Gauss-Newton and approximate inference for Bayesian deep learning.
The generalized Gauss-Newton method is an optimization method that is used in several popular Bayesian deep learning algorithms.
In particular, algorithms that combine the Gauss-Newton method with the Laplace and Gaussian variational approximation have recently led to state-of-the-art results in Bayesian deep learning.
While the Laplace and Gaussian variational approximation have been studied extensively, their interplay with the Gauss-Newton method remains unclear.
For example, we know that both approximate inference methods compute a Gaussian approximation to the posterior.
However, it is not clear how the Gauss-Newton method impacts the underlying probabilistic model or posterior approximation.
Additionally, recent criticism of priors and posterior approximations in Bayesian deep learning further urges the need for a deeper understanding of practical algorithms.

The individual analysis of the Gauss-Newton method and Laplace and Gaussian variational approximations for neural networks provides both theoretical insight and new practical algorithms.
We find that the Gauss-Newton method simplifies the underlying probabilistic model significantly.
In particular, the combination of the Gauss-Newton method with approximate inference can be cast as inference in a linear or Gaussian process model.
We find that the Gauss-Newton method turns the original model locally into a linear or Gaussian process model.
The Laplace and Gaussian variational approximation can subsequently provide a posterior approximation to these simplified models.
This new disentangled understanding of recent Bayesian deep learning algorithms also leads to new methods:
first, the connection to Gaussian processes enables new function-space inference algorithms.
Second, we present a marginal likelihood approximation of the underlying probabilistic model to tune neural network hyperparameters.
Finally, the identified underlying models lead to different methods to compute predictive distributions.
In fact, we find that these prediction methods for Bayesian neural networks often work better than the default choice and solve a common issue with the Laplace approximation.

\end{abstract}


\par
\cleardoublepage
\chapter*{Mathematical Notation and Abbreviations}
\markboth{Mathematical Notation and Abbreviations}{Mathematical Notation and Abbreviations}
\addcontentsline{toc}{chapter}{Mathematical Notation and Abbreviations}
\begin{table}[ht]
  \centering
  \begin{tabular}{l l}
  Symbol & Explanation\\
  \midrule
  $\R, \R_+$ & set of real and positive real numbers\\
  $x$ & scalar variable \\
  $\vv$ & vector variable with scalar entries $v_i$ \\
  $\mM$ & matrix variable with scalar entries $\mM_{ij}$ \\
  $\tT$ & tensor variable with scalar entries $\tT_{ijk}$\\
  $f(\vx;\param)$ & function mapping $\vx$ to some output parameterized by $\param$, for example a neural network\\
  $\nabla_\param f(\param)$ & the gradient with individual entries $ \sqr{ \nabla_\param f(\param) }_i = \frac{\partial f(\param)}{\partial \theta_i}$\\
  $\nabla_{\param \param}^2 f(\param)$ & the Hessian with entries $\sqr{ \nabla_{\param \param}^2 f(\param) }_{ij} = \frac{\partial^2 f(\param)}{\partial \theta_i \partial \theta_j}$\\
  $\nabla_\param f(\param_*)$ & gradient evaluated at $\param = \param_*$, same applies for Hessian\\
  $\langle \cdot, \cdot \rangle$ & scalar, vector, or Frobenius inner product depending on the context\\
  $\gauss (\vmu, \mSigma)$ & multivariate normal distribution with mean $\vmu$ and covariance $\mSigma$\\
  $\gauss (\param; \vmu, \mSigma)$ & $\param$ is distributed according to $\gauss (\vmu, \mSigma)$.\\
  \vspace{1em}\\
  Abbreviation & Meaning\\
  \midrule
  GGN, GN & (generalized) Gauss-Newton \\
  GVA & Gaussian variational approximation \\
  GP & Gaussian process \\
  GLM & generalized linear  model \\
  GGPM & generalized Gaussian process model \\
  BLR & Bayesian linear regression (model) \\
  VOGGN & variational online generalized Gauss-Newton\\
  OGGN & online generalized Gauss-Newton\\
  LGVA & linearized Gaussian variational approximation\\
  \end{tabular}
\end{table}

\maketoc
\chapter{Introduction}
\label{sec:introduction}
The field of \emph{machine learning} deals with algorithms that teach computers to make predictions or take actions in novel scenarios based on past experience.
In the setting of \emph{supervised learning}, the past experience consists of observed data points comprising inputs and labels.
For example, the input could be an image of an object and the label its name.
A machine learning algorithm can teach, or \emph{train}, the computer to predict the label of unseen images using a \emph{learned} model.
In the \emph{probabilistic machine learning} framework, we model uncertainties about our choice of model, i.e., we don't restrict ourselves to a single bet.
Therefore, \emph{inference} in this framework is about identifying a distribution over models that explain the past experience well and therefore generalize to future observations.
Having a distribution instead of a single model enables uncertainty quantification and model comparison.

Apart from the data, a probabilistic machine learning algorithm comprises two key components:
a probabilistic model and an inference algorithm.
In this work, we focus on probabilistic models of neural networks and scalable inference algorithms for such models.
We disentangle common inference algorithms that combine \emph{approximate inference} with the \emph{generalized Gauss-Newton} (GGN) optimization method and investigate how both affect the underlying probabilistic model.
We show that, locally, these techniques simplify the probabilistic model drastically.
Investigating these simplified underlying models helps to improve our understanding of Bayesian deep learning.
We further exploit this understanding to enhance Bayesian deep learning.
In fact, the findings presented here lead to more accurate predictions of neural network models inferred with approximate inference methods.
The disentanglement of the Gauss-Newton method and approximate inference further enables novel ways to compute the posterior and posterior predictive of a neural network in the function-space and gives rise to a marginal likelihood approximation.
Further, we identify a new variational inference algorithm and provide experimental support for our theoretical results and hypotheses.

\section{Probabilistic Models and Inference}%
\label{sec:prob_model_infer}

More formally, the combination of data and a corresponding probabilistic model of the data can be described as follows.
In the supervised setting, we are given a dataset $\data = \{(\vx_i, \vy_i)\}_{i=1}^N$ of $N$ independent and identically distributed input $\vx_i$ and label $\vy_i$ pairs.
For now, let both input and output be some abstract quantities.
In the previous example, $\vx_i$ corresponds to an image while $\vy_i$ denotes the corresponding object name.
A probabilistic model consists of \emph{prior} and \emph{likelihood}.
In the case of probabilistic neural networks, we usually pose a \emph{prior} over the parameters $\param$, i.e., $p(\param)$.
The \emph{likelihood} of observing a label $\vy$ for a given parameter $\param$ and input $\vx$ can be written as $p(\vy \given \param, \vx)$.
Since all data points are i.i.d., we write $p(\data \given \param)$ for the product of all likelihoods over the entire dataset.

Turning to the problem of learning, we deal with the \emph{maximum a posteriori} (MAP) estimate and Bayesian inference.
In Bayesian inference, we compute the \emph{posterior distribution} $p(\param \given \data)$ over the parameter using Bayes rule
\begin{equation}
  \label{eq:bayes_rule}
  p(\param \given \data) = \frac{p(\data \given \param) p(\param)}{\int p(\data \given \param) p(\param) d\param} = \frac{p(\data, \param)}{p(\data)},
\end{equation}
where the normalization constant in the denominator is called the \emph{marginal likelihood}.
In contrast to Bayesian inference, MAP estimation solely captures the mode, i.e., the most probable parameter, of the posterior.
Ignoring the normalization constant, this can be achieved by maximizing the joint distribution:
\begin{equation}
  \label{eq:MAP}
  \param_\textrm{MAP} = \argmax_\param p(\data, \param).
\end{equation}
MAP estimation is typically much easier and computationally convenient while Bayesian inference is in many cases intractable due to the integration.
One could distinguish the two types of inference as follows:
Bayesian inference is \emph{learning by integration} while MAP estimation is \emph{learning by optimization}.
Computationally, optimization is much more convenient than integration.

Ignoring the computational burden, we would often prefer Bayesian inference since the posterior distribution captures more information than the MAP estimate.
While the MAP estimate provides only \emph{local} information at $\param_\textrm{MAP}$, the posterior distribution carries \emph{global} information due to the integration over the parameter space captured in the marginal likelihood $p(\data)$.
In the case of linear regression, it is clear that Bayesian inference has advantages over the MAP estimate:
we can quantify uncertainty of our predictions and compare the marginal likelihood between models to obtain the one explaining the data best.
Nonetheless, computational feasibility is a key requirement of machine learning algorithms.
Therefore, the MAP estimator is often the standard choice as it provides good results at fraction of the cost and effort.

\section{Bayesian Deep Learning}%
\label{sec:BDL}

Bayesian neural networks are probabilistic models where the likelihood is parameterized by a neural network and the prior is a distribution over the neural network parameters.
Bayesian inference in these models is particularly challenging and therefore MAP estimation has played the major role in the past successes of \emph{deep learning}~\cite{deng2014deep, lecun2015deep}.
In contrast, the alternative strand of research termed \emph{Bayesian deep learning} deals with \emph{approximate Bayesian inference} for probabilistic neural network models.
Instead of obtaining an exact posterior of the Bayesian inference problem, an approximation to the posterior is constructed~\cite{blundell2015weight, khan2018fast, wang2016towards, zhang2018noisy}.
The key aspect of Bayesian deep learning is to maintain computational and performance advantages of deep learning while providing parameter uncertainties that enable predictive uncertainties.

Bayesian deep learning algorithms rely on approximate inference techniques as well as scalable optimization algorithms.
The interplay of both enables to approximate the posterior distribution efficiently.
In particular, the combination of the generalized Gauss-Newton method from the optimization literature and Gaussian posterior approximations has recently made Bayesian deep learning competitive with traditional deep learning and enabled new applications based on uncertainty~\cite{khan2018fast, osawa2019practical, ritter2018scalable, zhang2018noisy}.
However, it is unclear how the combination of these approximations impacts the underlying inference problem.
In this work, we address this problem by disentangling algorithms that make use of Gaussian posterior approximations and the generalized Gauss-Newton method in Bayesian deep learning.
Decoupling the individual methods and analyzing them individually can potentially allow to understand the algorithms better, improve them, and fix existing pathologies.

\section{Outline of the Thesis}%
\label{sec:outline}


In \autoref{sec:background}, the necessary background on probabilistic neural network models, approximate inference, and the generalized Gauss-Newton method is introduced.
In \autoref{sec:laplace}, we disentangle the combination of the Laplace approximation with the generalized Gauss-Newton method for optimization.
This allows to understand approximate inference for neural networks with linear and Gaussian process models.
Further, we introduce new methods to compute the posterior predictive and marginal likelihood approximation of neural network models.
Chapter~\ref{sec:vi} introduces a new variational inference algorithm and, in the same spirit as Chapter~\ref{sec:laplace}, describes approximate inference for neural networks via simpler linear and Gaussian process models.
Chapter~\ref{sec:experiments} provides experiments that explain and complement the theoretical connections:
we show how to tune hyperparameters using the marginal likelihood and that the posterior predictive is in fact greatly influenced by the application of the Gauss-Newton method.
The prediction methods introduced here fix a common problem with the Laplace approximation.
The kernel of the identified Gaussian process formulation is further used to explain neural network predictions.
Lastly, we discuss related and future work and conclude the thesis in Chapter~\ref{sec:discussion}.


\chapter{Background}
\label{sec:background}
In this chapter, we briefly introduce the necessary concepts and theory to follow the rest of this work.
We first introduce probabilistic neural networks as the underlying models that we want to infer.
In particular, we introduce likelihoods for supervised learning problems and discuss the choice of prior.
Bayesian deep learning combines approximate inference methods and optimization algorithms, both of which we introduce in the end of this chapter.
On the inference side, we introduce the Laplace and Gaussian variational approximation and relate them to each other.
Lastly, we introduce the \emph{generalized Gauss-Newton} approximation to the Hessian from the optimization literature.
This enables scalable approximate inference.

Formally, we denote a neural network as $\vf(\vx; \param): \R^D \times \R^P \rightarrow \R^K$ that maps an input $\vx \in \R^D$ to output $\vlogits \in \R^K$ with parameters $\param \in \R^P$.
We do not restrict ourselves to any form of neural network.
In fact, any parametric function that is differentiable in $\param$ at least once can be used.
In some cases, we have a scalar output, i.e., $f(\vx;\param)=f \in \R$.

The dataset in a supervised learning scenario is given as pairs of inputs $\vx_i \in \R^D$ and labels $\vy_i \in \R^K$.
As introduced in Chapter~\ref{sec:introduction}, the entire dataset of size $N$ is then given by $\data=\{(\vx_i, \vy_i)\}_{i=1}^N$.
We assume that all data points are drawn independently and from the same distribution.
Therefore, we have the following likelihood in a neural network model:
\begin{equation}
  \label{eq:likelihood_factorized}
  p(\data \given \param) = \prod_{i=1}^N p(\vy_i \given \vf(\vx_i;\param)).
\end{equation}
In deep learning, we obtain a MAP estimate by optimizing the corresponding objective in \autoref{eq:MAP}.
For computational reasons, we maximize the log joint distribution which leads to the more convenient objective
\begin{equation}
  \label{eq:MAP_log}
  \param_\textrm{MAP} = \argmax_\param \sum_{i=1}^N \log p(\vy_i \given \vf(\vx_i ;\param)) + \log p(\param),
\end{equation}
which is also known as the \emph{empirical risk minimization} objective.
In the ERM setting, the likelihood acts as a loss per data point and the prior can be understood as a regularizer.
Next, we specify a particular family of likelihoods that is sufficient for deep learning and possesses useful theoretical properties.
After that, we introduce approximate Bayesian inference and the generalized Gauss-Newton method.

\section{Probabilistic Neural Networks for Supervised Learning}
\label{sec:GNMs}
We have introduced neural networks that map an input data point to an output value.
A probabilistic neural network model additionally consists of a likelihood and a prior.
We introduce a family of likelihoods, in particular, those of \emph{generalized linear models} (GLMs), that give rise to common losses used in deep learning.
These exponential family likelihoods possess simplifying theoretical properties that make later results more interpretable.
In the end of the section, we discuss the choice of prior.

\begin{figure}[t]
    \centering
    \includegraphics[height=5em]{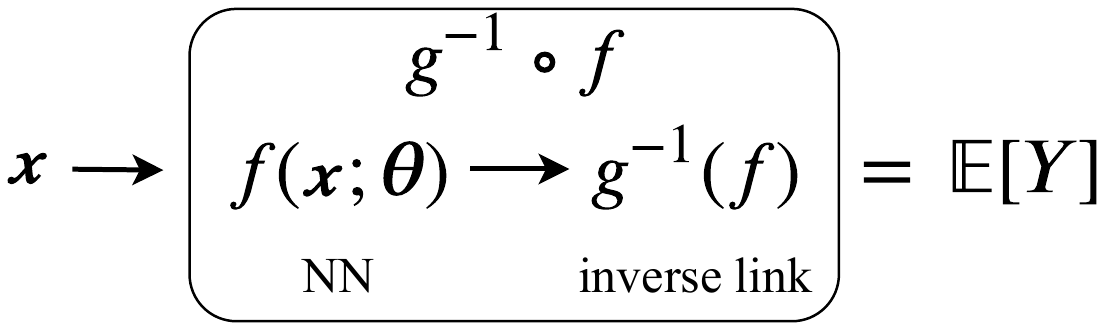}
    \caption{Illustration of the composition of neural network and inverse link function.
    The inverse link function gives the mean of the modelled response variable $Y$.}
    \label{fig:GNM_illustration}
\end{figure}

Generalized linear model likelihoods specify a distribution over the output labels.
In particular, we model a response random variable $Y$ or vector $\mY$ for a given input and neural network parameter.
GLM likelihoods conveniently express the mean of the response via an invertible \emph{link function} $g(\cdot)$, i.e., $\myexpect \sqr{ \mY } = \vg^{-1} \rnd{ \vf }$.
Further, the derivatives of the log-likelihood with respect to the function $\vf$ are directly related to the moments of $\mY$.
This property makes the theoretical developments more intuitive and insightful.
The likelihoods of generalized linear models are restricted to exponential family distributions~\cite{murphy2012machine}.
Every exponential family can be written in the following \emph{natural form} as
\begin{equation}
\label{eq:glm_likelihood}
  p(\vy \given \vf) = h(\vy) \exp{\crl{\langle T(\vy), \vf \rangle - A \rnd{\vf}}},
\end{equation}
where $h(\vy)$ is the base measure, $T(\vy)$ the sufficient statistics, $A(\vf)$ the log-cumulant, and $\vf$ the natural parameter.
We restrict ourselves to forms where $T(\cdot)$ is the identity, i.e., $T(\vy) = \vy$.
Therefore, we have
\begin{equation}
  \label{eq:glm_likelihood_id}
  p(\vy \given \vf) = h(\vy) \exp{\crl{ {\langle \vy, \vlogits \rangle - A \rnd{\vlogits}}}}.
\end{equation}

The derivative of the log likelihood with respect to the natural parameter $\vf$ is particularly convenient.
The first derivative forms a residual between the observed label and the mean of the response variable and therefore specifies the link function.
The second derivative directly relates to the variance of the modelled response variable.
Further, we identify the mean of the response variable and therefore the inverse link function as the derivative of the log cumulant $A(\vf)$.
The following Lemma formalizes these properties.

\begin{restatable}{lemma}{GLMS}
\label{thm:exp_derivatives}
Let $p(\vy \given \vf)$ be an exponential family distribution of the form in \autoref{eq:glm_likelihood_id} and let $\mY$ denote the corresponding random variable.
The first and second derivative of the log likelihood take the simple form
\begin{align}
\label{eq:glm_derivatives}
  \nabla_\vf \log p(\vy \given \vf) =
  &=   \vy - \nabla_\vf A(\vf)
  =  \vy - \myexpect \sqr{\mY}=  \vy - \vg^{-1} (\vlogits)
  =: \vr(\vy, \vlogits),\\
  \nabla_{\vf \vf}^2 \log p(\vy \given \vlogits) &= -\nabla_{\vf \vf}^2 A(\vf)
  = - \myvar \sqr{ \mY }
  =:- \mLambda(\vlogits),
\end{align}
where we defined the residual $\vr(\vy, \vf) \in \R^K$ and second derivative, or Hessian, of the negative log likelihood $\mLambda(\vf) \in \R^{K\times K}$.
We have thus identified the inverse link that gives the mean of the response variable as the first derivative of the log cumulant.
In the case of an \emph{overdispersed} exponential family~\cite{murphy2012machine}, the variance of the response variable is further scaled by the dispersion parameter $\sigma^2$ while the derivatives are divided by it.
We have this case, for example, in a Gaussian likelihood (see \autoref{tab:glms}).
\end{restatable}

\begin{table}
  \centering
  \begin{tabular}{l l l l l l}
    Distribution & $\sigma^2$ & $\myexpect \sqr{\mY} = \vg^{-1}(\vf)$ & $\vr(\vy, \vlogits)$  & $\myvar \sqr{\mY}$ & $\mLambda(\vf)$ \\
    \midrule
    Gaussian  & $\sigma^2$ & $f$ & $\sigma^{-2}(y - f)$ & $\sigma^2$ & $\sigma^{-2}$ \\
    Multivariate Gaussian  & $\mSigma$ & $\vf$ & $\mSigma^{-1} (\vy - \vf)$ & $\mSigma$ & $\mSigma^{-1}$ \\    Bernoulli & 1 & $ \frac{1}{1+e^{-f}} = \sigmoid(f)$ & $y - \sigmoid(f)$ & $\sigmoid(f) (1 - \sigmoid(f))$ & $\myvar \sqr{ \mY }$ \\
    Categorical & 1 & $\softmax(\vf) = \vp(\vf)$ & $\vy - \vp(\vf)$ & $\diag(\vp) - \vp \vp^\top$ & $\myvar \sqr{ \mY }$ \\
    Poisson & 1 & $ \exp (f)$ & $y - \exp (f)$ & $\exp(f)$ & $\myvar \sqr{ \mY }$\\
  \end{tabular}
\caption{Most common likelihoods of generalized linear models.
For each distribution, we list the corresponding \emph{dispersion} parameter, the mean response, the residual, the variance of the response, and the Hessian.
In the context of deep learning, the Bernoulli and categorical likelihoods are common due to their application to classification problems.}
\label{tab:glms}
\end{table}

A proof of the Lemma can be found in Appendix~\ref{cha:proof_of_glm_log_likelihood_derivatives}.
In \autoref{tab:glms}, we list the most notable examples of likelihoods.
In this work, we mostly deal with the Gaussian univariate and Bernoulli likelihood.
Both naturally extend to their multivariate counter-parts, the multivariate Gaussian and categorical distribution.
In the empirical risk minimization perspective, a Gaussian likelihood yields a least-squares loss used for regression problems.
The Bernoulli and categorical distributions give rise to the commonly used \emph{cross-entropy} loss for classification~\cite{goodfellow2016deep}.

As prior we choose a multivariate Gaussian in line with most works on Bayesian neural networks~\cite{blundell2015weight, khan2018fast, zhang2018noisy}.
In the MAP estimation framework, a Gaussian prior corresponds to $\ell_2$ regularization commonly known as \emph{weight decay} in deep learning~\cite{bishop2006pattern, goodfellow2016deep, murphy2012machine}.
We therefore have
\begin{equation}
  \label{eq:prior}
\param \sim p(\param) = \gauss \rnd{\vmu_0, \mSigma_0},
\end{equation}
where $\vmu_0 \in \R^P$ is the mean and $\mSigma_0 \in \R^{P\times P}$ the covariance.
For neural networks, the mean is typically chosen to be zero and the covariance is either diagonal, i.e., $\vsigma_0^2 \in \R_+^P$, or scalar $\sigma_0^2 \in \R_+$.
That is because it is nontrivial to specify a prior over the parameters of a neural network and the Gaussian is a comparably agnostic default choice~\cite{mackay1992bayesian, nalisnick2018priors, neal2012bayesian}.
It is important to note that the prior should be understood as a distribution in Bayesian inference and as a regularizer in MAP estimation.
In \autoref{sec:laplace}, we pick up on this distinction in the context of approximate Bayesian inference.

\section{Approximate Bayesian Inference}%
\label{sec:ABI}

Exact Bayesian inference in a probabilistic neural network model is typically intractable.
Therefore, we need to resort to approximate inference methods that provide scalable and computable alternatives to the marginal likelihood.
In the context of networks, the most scalable and practical techniques are variants of the \emph{Laplace}~\cite{foresee1997gauss, mackay1992bayesian, ritter2018scalable} and \emph{Gaussian variational} (GVA) approximation~\cite{blundell2015weight, graves2011practical, khan2018fast, zhang2018noisy}.
In this work, we therefore focus on these two approximations.
Both, the Laplace approximation and the GVA turn the problem of integration into a problem of optimization followed by a simple closed-form integration.
Other approximate inference methods such as \emph{Markov chain Monte Carlo}~\cite{neal1993probabilistic, neal2012bayesian} or \emph{expectation} and \emph{belief propagation}~\cite{minka2001expectation,pearl1986fusion} are typically not as scalable and are therefore rarely used for neural networks.

The Laplace and Gaussian variational approximation both construct a Gaussian approximation to the true posterior distribution.
The Laplace uses a second-order approximation of the log joint distribution to approximate the marginal likelihood.
In contrast, the GVA maximizes the variational lower bound to the marginal likelihood.
Both methods give rise to a posterior approximation $q (\param)$ and we have
\begin{equation}
  \label{eq:ABI}
q(\param) = \gauss \rnd{\param ; \vmu, \mSigma }  \approx p(\param \given \data),
\end{equation}
where $\vmu \in \R^P$ and $\mSigma \in \R^{P \times P}$ are the free parameters.
Typically, the variational approximation is preferred because it can capture the shape of the posterior better~\cite{bishop2006pattern, murphy2012machine}.
After introducing both methods, we briefly show how they are related.

\subsection*{The Laplace Approximation}%
\label{sub:laplace}
The Laplace approximation is a heuristic that fits a Gaussian distribution locally at the MAP estimate (\autoref{eq:MAP_log}).
Therefore, the Laplace approximation comprises two steps:
first, we obtain the MAP estimate $\param_\textrm{MAP}$ using an optimization method and then we approximate the log joint distribution to the second-order using a Taylor expansion.
The second-order approximation allows us to compute the marginal likelihood in closed form.
The Laplace approximation therefore provides an approximation to the posterior and the marginal likelihood.

Since the gradient at the MAP is zero, i.e. $\nabla_\param \log p(\data, \maparam) = \vzero$, the second-order Taylor approximation at this point is simple.
We have
\begin{align}
  \label{eq:taylor_laplace}
\log p(\data, \param) &\approx \log p(\data, \maparam) + \frac12 \rnd{\param - \maparam}^\top \nabla_{\param \param}^2 \log p(\data, \maparam)  \rnd{\param - \maparam}.
\end{align}
Then, computing the marginal likelihood of the approximated model has a closed-form solution.
Using the normalization properties of a multivariate normal distribution, we obtain the Laplace approximation to the marginal likelihood
\begin{align}
  \label{eq:laplace_marglik}
  p(\data) &\approx p(\data, \maparam) \rnd{ 2\pi }^{ \frac{P}{2} }  \det(-\nabla_{\param \param}^2 \log p(\data, \maparam))^{-\frac12}.
\end{align}
Exponentiating the approximation to the joint distribution in \autoref{eq:taylor_laplace} and dividing it by the marginal likelihood, we identify the Laplace approximation as a Gaussian distribution with mean and covariance given by
\begin{equation}
  \label{eq:laplace}
  \vmu = \maparam \quad \textrm{and} \quad \mSigma = [-\nabla_{\param \param}^2 \log p(\data, \maparam)]^{-1}.
\end{equation}

The Laplace approximation is a practical method for approximate Bayesian inference because we only need to find a MAP estimate and compute the Hessian of the log joint distribution at that estimate.
Deep learning optimizers provide effective means to obtain MAP estimates.
For large networks, however, computing the Hessian for a neural network model is permissively complex in terms of storage and computation.
Further, it might even give undesirable results:
the Hessian at the MAP we obtain is not necessarily a positive definite matrix and might therefore not be invertible~\cite{sagun2016singularity, sagun2017empirical}.
Typically, this problem is tackled by optimization methods that guarantee a positive semi-definite Hessian approximation.
For example, Gauss-Newton methods~\cite{hartley1961modified, khan2018fast, ritter2018scalable} ensure an invertible Hessian approximation.
In \autoref{sec:optim}, we introduce the generalized Gauss-Newton method that is commonly employed.

\subsection*{The Gaussian Variational Approximation}%
\label{sub:gva}

In \emph{variational inference}, we minimize the \emph{Kullback-Leibler} (KL) divergence of the true posterior $p(\param \given \data)$ from an approximation distribution $q(\param)$~\cite{blei2017variational, hoffman2013stochastic}.
Let $\gP$ be a family of distributions that we choose as our posterior approximating family.
Then, the variational approximation is given by the following optimization problem:
\begin{equation}
\begin{split}
  q_*(\param)
  &= \argmin_{q \in \gP} \KL \sqr{q(\param) \| p(\param \given \data)}
  = \argmin_{q \in \gP} \int q(\param) \log \rnd{ \frac{q(\param)p(\data)}{p(\data \given \param) p(\param)} } d \param\\
  &= \argmin_{q \in \gP} \int q(\param) \log \rnd{ \frac{q(\param)}{p(\data, \param)}  } d \param + \log p(\data).\label{eq:vi_objective}
\end{split}
\end{equation}
If the approximating family $\gP$ contains the true posterior distribution, the variational approximation is exact and naturally incurs no divergence.
Since the KL divergence is non-negative, the first term in \autoref{eq:vi_objective} stands in special relation to the log marginal likelihood giving rise to the \emph{evidence lower bound} (ELBO):
\begin{equation}
  \label{eq:elbo}
  \log p(\data) \geq \myexpect_q \sqr{\log \frac{p(\data, \param)}{q(\param)} } = \textrm{ELBO}(q).
\end{equation}
Maximizing the ELBO is hence equivalent to minimizing the KL divergence of the true posterior from the approximating distribution.

The Gaussian variational approximation (GVA) is a particular instance of variational inference where $\gP$ is the family of multivariate Gaussian distributions.
In particular, the approximating distribution $q(\param)$ is restricted to the Gaussian distribution $\gauss(\param; \vmu, \mSigma)$ as stated in \autoref{eq:ABI}.
Therefore, the problem of optimizing over a family of distributions turns into optimizing the parameters of a distribution of fixed form.
The ELBO for a probabilistic neural network model with Gaussian prior (\autoref{sec:GNMs}) can therefore be expressed in terms of the parameters $\vmu, \mSigma$.
Further, it is convenient to split the ELBO into an expected log likelihood and a KL divergence term:
\begin{align}
  \label{eq:gva_elbo}
  \textrm{ELBO}(\vmu, \mSigma)
  =\myexpect_{\gauss(\vmu, \mSigma)} \sqr{\log p(\data \given \param)} - \KL \sqr{\gauss(\vmu, \mSigma) \| \gauss(\vmu_0, \mSigma_0)}.
\end{align}
The KL divergence has a closed form solution since both prior and approximating distribution are Gaussian.
This form of the ELBO provides the basis of variational inference in deep learning~\cite{blundell2015weight, khan2018fast, zhang2018noisy}.
The expected log likelihood term of the ELBO seems to be complicated to optimize.
However, a neat relation between derivatives with respect to the parameters $(\vmu, \mSigma)$ and realized samples $\param_s$ of the GVA exists~\cite{opper2009variational} and is due to Bonnet and Price~\cite{bonnet1964transformations, price1958useful}.
We can simplify the derivative with respect to parameters $\vmu$ and $\mSigma$ of the expected log likelihood as follows:
\begin{align}
  \nabla_\vmu \myexpect_{\gauss(\vmu, \mSigma)} \sqr{\log p(\data \given \param)} &= \myexpect_{\gauss(\vmu, \mSigma)} \sqr{\nabla_\param \log p(\data \given \param)} \label{eq:bonnet_mean} \\
  \nabla_\mSigma \myexpect_{\gauss(\vmu, \mSigma)} \sqr{\log p(\data \given \param)}
  &= \frac12 \myexpect_{\gauss(\vmu, \mSigma)} \sqr{\nabla_{\param \param}^2 \log p(\data \given \param)}
  \label{eq:bonnet_var}
\end{align}
Therefore, taking gradients with respect to the variational parameters can be as simple as sampling and taking the gradient with respect to individual samples.
Further, these equations make clear that the GVA is also limited to a second-order approximation.
However, the GVA maintains a global \emph{view} of the loss due to the expectation and is therefore more powerful than the Laplace approximation.

\subsection*{Relation between Laplace and Variational Approximation}%
\label{sub:relation_gva_vi}

Following \citet{opper2009variational}, we compare the optimality criteria of the Laplace and Gaussian variational approximation.
For the Laplace approximation, we have the conditions
\begin{equation}
  \label{eq:laplace_stationary}
  0 = \nabla_\param \log p(\data, \vmu) \quad \textrm{and} \quad \mSigma^{-1} = -\nabla^2_{\param \param} \log p(\data, \vmu).
\end{equation}
For the Gaussian variational approximation, we can devise very similar conditions.
First, we obtain the stationarity conditions by differentiating with respect to the variational parameters in \autoref{eq:gva_elbo}.
The stationarity condition for the mean is simple and for the inverse covariance matrix, we have
\begin{equation}
  \label{eq:vi_cov_stationarity}
  \nabla_\mSigma \textrm{ELBO}(\vmu, \mSigma) = 0
  \rightarrow \mSigma^{-1}= 2\nabla_\mSigma \myexpect_{\gauss(\vmu, \mSigma)} \sqr{-\log p(\data, \param)}.
\end{equation} 
Now, we apply equalities of \autoref{eq:bonnet_mean} and \autoref{eq:bonnet_var} to both stationarity conditions.
Note that these equalities hold not only for the expectation over a likelihood but also the joint distribution~\cite{opper2009variational}.
We obtain the GVA stationarity conditions
\begin{equation}
  \label{eq:vi_stationarity}
  0 = \myexpect_{\gauss(\vmu, \mSigma)} \sqr{\nabla_\param \log(\data, \param)}
  \quad
  \textrm{and}
  \quad
  \mSigma^{-1} = \myexpect_{\gauss(\vmu, \mSigma)} \sqr{-\nabla_{\param \param}^2 \log p(\data, \param)}.
\end{equation}

The relation between \autoref{eq:laplace_stationary} and \autoref{eq:vi_stationarity} highlights the difference between both approximations:
while the Laplace approximation is only defined locally at the MAP, the variational approximation holds globally~\cite{opper2009variational}.
The relation suggests that the difference lies in sampling parameters versus fixed parameters.
Specifically in practice, the variational inference stationarity in \autoref{eq:vi_stationarity} does not hold exactly but for $S$ Monte Carlo samples from the variational approximation, i.e., $\param^{(1)},\dots,\param^{(S)} \sim \gauss (\vmu, \mSigma)$:
\begin{equation}
  \label{eq:vi_stationarity_mc}
  0 = \frac1S \sum_{s=1}^S \nabla_\param \log(\data, \param^{(s)})
  \quad
  \textrm{and}
  \quad
  \mSigma^{-1} = \frac1S \sum_{s=1}^S -\nabla_{\param \param}^2 \log p(\data, \param^{(s)}).
\end{equation}
The main insight of this section is that both approximations limit themselves to the first two moments of the negative log joint distribution.
The variational approximation provides a more global view of the joint distribution by stochasticity in the parameters.

\section{The Generalized Gauss-Newton Method}
\label{sec:optim}

For approximate Bayesian inference in neural networks, second-order derivatives of the log likelihood are typically required.
In particular, this is the case for the Laplace and Gaussian variational approximations.
As mentioned in \autoref{sec:ABI} however, the Hessian is permissively expensive to compute and might even be singular or undefined.
Therefore, the current state-of-the-art Bayesian deep learning algorithms rely on approximate second-order optimization methods~\cite{khan2018fast, osawa2019practical, zhang2018noisy}.
These methods are both scalable and guarantee a positive semi-definite approximation to the Hessian~\cite{bottou2018optimization, martens2014new, schraudolph2002fast}.
In particular, the \emph{generalized Gauss-Newton approximation} (GGN) is used extensively for the Laplace and Gaussian variational approximation~\cite{foresee1997gauss, khan2018fast, martens2015optimizing, ritter2018scalable, zhang2018noisy}.
The GGN is a positive semi-definite approximation to the Hessian.
In practice, the diagonal~\cite{graves2011practical, khan2018fast} or a Kronecker factorization~\cite{ritter2018scalable, zhang2018noisy} of the GGN is often used.
Here, we work with a full GGN approximation to draw conclusions for the other special cases.

The generalized Gauss-Newton method allows us to compute an approximation to the second derivative of the log-likelihood.
The derivative of the log prior in the Laplace approximation and the KL divergence from the prior in the GVA have a closed form and do not require an approximation.
The log likelihood takes the following form in the case of probabilistic neural network models.
According to \autoref{eq:likelihood_factorized}, we have
\begin{equation}
  \label{eq:log_likelihood_summed}
  \log p(\data \given \param) = \sum_{i=1}^N \log p(\vy_i \given \vf(\vx_i;\param)),
\end{equation}
where the likelihood is a generalized linear model likelihood (\autoref{sec:GNMs}).
To take the first and second derivative with respect to the parameter, we apply the chain rule.
\autoref{fig:GNM_illustration} illustrates that we can first differentiate with respect to $\vf$ and then with respect to the parameters.
We define the \emph{Jacobian matrix} $\mJ(\vx;\param) \in \R^{K\times P}$ of $\vf(\vx;\param)$ with respect to the parameters, and the \emph{Hessian tensor} $\tH(\vx;\param) \in \R^{K \times P \times P}$ of second derivatives as
\begin{equation}
  \label{eq:Jacobian}
  \sqr{ \mJ(\vx;\param) }_{ij} = \frac{\partial f_i(\vx;\param)}{\partial \theta_j} \quad \textrm{and} \quad \sqr{ \tH(\vx;\param) }_{ijk} = \frac{\partial^2 f_i(\vx;\param)}{\partial \theta_j \partial \theta_k} ,
\end{equation}
where we assumed the function is twice differentiable for now.
Using the properties of the first and second derivative of the GLM log likelihoods in \autoref{thm:exp_derivatives} and \autoref{tab:glms}, we obtain the gradient
\begin{equation}
  \label{eq:grad_loglik}
  \nabla_\param \log p(\vy \given \vlogits(\vx;\param)) = \mJ(\vx;\param)^\top \nabla_\vf \log p(\vy \given \vf) = \mJ(\vx;\param)^\top \vr(\vy,\vlogits).
\end{equation}
Similarly, the Hessian of the log likelihood can be computed using the chain rule and gives
\begin{equation}
\begin{split}
  \label{eq:hessian_loglik}
  \nabla_{\param \param}^2 \log p(\vy \given \vlogits(\vx;\param))
  &= \tH(\vx;\param)^\top \nabla_\param \log p(\vy \given \vf) + \mJ(\vx;\param)^\top \nabla_{\vf \vf}^2 \log p(\vy \given \vf) \mJ(\vx;\param) \\
  &= \tH(\vx;\param)^\top \vr(\vy,\vlogits) - \mJ(\vx;\param)^\top \mLambda(\vf) \mJ(\vx;\param).
\end{split}
\end{equation}

The first derivative is tractable and efficiently implemented in neural networks using \emph{backpropagation}~\cite{goodfellow2016deep}.
The second derivative with respect to the parameters is problematic because we need to differentiate the neural network with respect to its large amount of parameters twice.
In fact, for some network architectures, for example, \texttt{ReLU} activation functions, the second derivative is not defined everywhere~\cite{zhang2018noisy}.
The generalized Gauss-Newton approximation to the Hessian simply removes the term that is intractable and is given by
\begin{equation}
  \label{eq:GGN}
  \nabla_{\param \param}^2 \log p(\vy \given \vf(\vx;\param)) \approx - \mJ(\vx;\param) \mLambda(\vlogits) \mJ(\vx;\param)^\top.
\end{equation}
This makes the Hessian tractable since we only need first order derivatives with respect to the neural network.
Further, it is always positive semi-definite and does not even require the existence of the neural network Hessian.

Assuming the neural network Hessian $\tH(\vx;\param)$ exists, the GGN approximation is exact in two cases:
either, all residuals are zero, i.e., $\forall (\vx, \vy) \in \data: \vr(\vy,\logits(\vx;\param)) = \vzero$, or the neural network Hessian $\tH(\vx;\param)$ is zero.
Although a neural network can potentially achieve zero residuals, it is both undesirable as it indicates \emph{overfitting} and impractical as this condition does not hold at initialization or during training.\footnote{For the Bernoulli and categorical likelihood, the residual can theoretically only tend towards zero.}
The neural network Hessian $\tH(\vx;\param)$ can only be zero everywhere if the neural network is linear.
Therefore, an alternative derivation of the GGN approximation to the Hessian starts from linearization of the neural network~\cite{bottou2018optimization, martens2014new}.
We define the first order Taylor approximation of the neural network around the expansion point $\param_*$ as
\begin{equation}
  \label{eq:lin_NN}
  \vf_\textrm{lin}^{\param_*}(\vx;\param) = \vf(\vx;\param_*) + \mJ(\vx;\param_*) (\param - \param_*),
\end{equation}
which gives us a linear function in the parameters $\param$ but not in the input $\vx$.
To compute the Hessian at $\param_*$, we therefore linearize the network at this parameter and then compute the Hessian.
This way, we recover the GGN approximation to the Hessian.
This derivation and reasoning for the GGN allows to understand its role in approximate inference better.

\chapter{From Neural Network to Gaussian Process with Laplace and Gauss-Newton}
\label{sec:laplace}
The Laplace approximation is often used as a baseline for Bayesian neural networks and in some cases achieves state of the art results~\cite{ritter2018online,ritter2018scalable,titterington2004bayesian}.
For neural networks, the diagonal or a Kronecker-factored generalized Gauss-Newton approximation to the Hessian is often used~\cite{ritter2018scalable}.
Here, we will work with the full GGN approximation to the Hessian~\cite{foresee1997gauss}.
We call the combination of Laplace and GGN the Laplace-GGN approximation~\cite{khan2019approximate}.
In this chapter, we disentangle the Laplace-GGN approximation.
In particular, we first apply the generalized Gauss-Newton method and then make use of the Laplace approximation.
Going forward, we analyze the individual steps of the Laplace-GGN and their impact on the underlying probabilistic model.
Interestingly, the underlying probabilistic model can be cast as a Gaussian process model and enables function-space inference for neural networks.

\begin{table}[t]
  \centering
  \begin{tabular}{c c c c c}
  \small{Bayesian NN} & & \small{Bayesian GLM} & & \small{Linear Regression} \\
  \vspace{-0.5em}\\
  $p(\param \given \data)$ & $\xrightarrow{\makebox[1.3cm]{\textrm{\small{GGN}}}}$ & $\tilde{p}(\param \given \data)$ & $\xrightarrow{\makebox[1.3cm]{\textrm{\small{Laplace}}}}$ & $\hat{p}(\param\given \data) = q(\param)$\\ \vspace{-0.8em}\\
  & &  $\updownarrow$ & & $\updownarrow$ \\
  \vspace{-1.3em}\\
  & & \small{GP model} & $\xrightarrow{\makebox[1.3cm]{\textrm{\small{Laplace}}}}$ & \small{GP regression}\\
  \end{tabular}
  
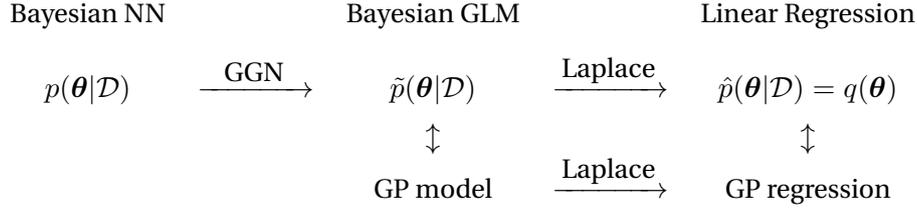
\captionof{figure}{Illustration of approximate inference with the Laplace-GGN method.
  We start with Bayesian neural network inference and, by applying the GGN and Laplace approximation, ultimately infer a Bayesian linear regression, or equivalently, GP regression model.
  As an intermediate step, the Bayesian neural network is turned into a generalized linear model due to the GGN.
  $\tilde{p}(\param \given \data)$ is the true posterior of the GLM and $\hat{p}(\param \given \data)$ is the true posterior of the Bayesian linear regression model which is equal to the Laplace-GGN approximation.}
  \label{fig:laplace_schema}
\end{table}

First, we apply the generalized Gauss-Newton method to the neural network model.
This gives rise to a generalized linear model due to the linearizing property of the GGN.
Equivalently, we can specify this model as a Gaussian process (GP) model.
Subsequently, we apply the Laplace approximation to either of these models.
We identify two simple models, the Bayesian linear and the Gaussian process regression model.
Inferring these models is then \emph{equivalent} to the Laplace-GGN approximation and facilitates a better understanding of the combination of Laplace and GGN approximation.
Figure~\ref{fig:laplace_schema} illustrates the steps and relationships established in this chapter.

After analyzing the Laplace and GGN approximation, we compute the posterior, posterior predictive, and marginal likelihood.
Due to the equivalence to Gaussian process inference, there are two ways for each of these quantities.
This can enable computational advantages in some cases, similar to the \emph{kernel trick}.
The Gaussian process viewpoint further enables model inspection and interpretability as we demonstrate in \autoref{sec:experiments}.

\section{From Neural Network to Generalized Linear Model}%
\label{sec:generalized_lienar_model}

We start with a probabilistic neural network model as specified in \autoref{sec:background}.
During optimization of the MAP objective, we have the iterate $\param_*$.
Typically, $\param_*$ is an estimate of the MAP but for our development this is not necessary.
The generalized Gauss-Newton approximation performs a linearization around the current parameter.
We recall \autoref{eq:lin_NN} from \autoref{sec:background}:
\begin{equation}
  \label{eq:lin_nn_lapsec}
  \linlogits^{\param_*}( \vx; \param) = \vf(\vx;\param_*) + \mJ(\vx;\param_*) (\param - \param_*).
\end{equation}
The linearization above changes our model locally.
We call $\tilde{p}(\param \given \data)$ the posterior of a generalized linear model.
Replacing the neural network $\vf(\vx;\param)$ by its linearized version in the probabilistic model, we then have
\begin{equation}
  \label{eq:glm_nn}
  \tilde{p}(\param \given \data)
  \propto p(\param) \prod_{i=1}^N p(\vy_i \given \linlogits^{\param_*}(\vx_i;\param)),
\end{equation}
which is a generalized linear model (GLM) since $\linlogits^{\param_*}(\vx;\param)$ is linear in the parameter $\param$.
The Jacobian can therefore be understood as a \emph{local} feature map of the inputs.
Equivalently, we can therefore transform this inference problem from the \emph{weight-space} to the \emph{function-space}~\cite{rasmussen2003gaussian}.
We obtain a Gaussian process prior by taking the expectation of the linearized neural network under the parametric prior $p(\param) = \gauss(\vmu_0, \mSigma_0)$.
The mean and covariance function of the Gaussian process prior are given by
\begin{align}
  \vm(\vx) &= \myexpect_{p(\param)} \sqr{ \linlogits^{\param_*}(\vx;\param) } = \linlogits^{\param_*}(\vx;\vmu_0) \label{eq:lap_gp_mean},\\
  \vkappa(\vx, \vx') &= \mycov_{p(\param)} \sqr{ \linlogits^{\param_*}(\vx;\param), \linlogits^{\param_*}(\vx';\param) } = \mJ(\vx;\param_*)^\top \mSigma_0 \mJ(\vx';\param_*),\label{eq:lap_gp_kernel}
\end{align}
which gives rise to the Gaussian process prior in common notation~\cite{rasmussen2003gaussian} as
\begin{equation}
  \label{eq:gp_glm_prior}
  \vf_{\textrm{GP}}(\vx) \sim \GP (\vm(\vx), \vkappa(\vx, \vx')).
\end{equation}
We define the posterior Gaussian process distribution by $\tilde{p}(\vf_\textrm{GP} \given \data)$ in line with the corresponding GLM.
We therefore have the following generalized Gaussian process model (GGPM) as termed by \citet{chan2011generalized}:
\begin{equation}
  \label{eq:glm_GP}
  \tilde{p}(\vf_{\textrm{GP}} \given \data) \propto p(\fgp) \prod_{i=1}^N p(\vy_i \given \fgp).
\end{equation}
The GLM and the GP model have the same posterior predictive since they specify the same prior over functions~\cite{rasmussen2003gaussian}.

So far, we have applied the GGN optimization method but no particular inference algorithm.
Nonetheless, the underlying probabilistic neural network models has turned into a generalized linear model.
The Jacobian, which constitutes the features of the GLM, remains fixed after this transformation.
The second step is now to apply the Laplace approximation to the GLM.
In this context, the Laplace approximation can therefore not really be understood as approximate inference of a neural network model.
In the next sections, we show that the Laplace approximation to the GLM or GGPM is equivalent to solving a Bayesian linear regression or GP regression model, respectively.
The exact posterior of these models then corresponds to the Laplace-GGN approximation.

\section{From Generalized Linear Model to Bayesian Linear Regression}%
\label{sec:gauss_laplace_blr}

The Laplace-GGN approximation is a Laplace approximation to the local generalized linear model in \autoref{eq:glm_nn}.
In this section, we show that the Laplace approximation to this generalized linear model can be understood as exact inference in a Bayesian linear regression (BLR) model.
The form of the Bayesian linear regression model is very intuitive and improves our understanding of the Laplace-GGN approximation.
To keep notation simple, we abbreviate all relevant quantities as follows:
we write $\vlogits(\vx):= \vf(\vx;\param_*)$, $\vg^{-1}(\vx):= \vg^{-1}(\vf(\vx_i;\param_*))$, $\mJ(\vx) := \mJ(\vx;\param_*)$, and $\mLambda(\vx) := \mLambda(\vf(\vx;\param_*))$.
We obtain the following result:

\begin{theorem}
\label{thm:laplace}
The Laplace-GGN approximation to a Bayesian neural network model, i.e., applying the Laplace approximation to the generalized linear model in \autoref{eq:glm_nn}, is equivalent to the exact posterior of a Bayesian linear regression model.
In particular, we obtain for the Laplace approximation that
\begin{equation}
  \label{eq:laplace_linmodel}
  q(\param) = \hat{p}(\param \given \data) \propto p(\param) \prod_{i=1}^N \gauss \rnd{\vy_i \given \vg^{-1}(\vx_i)+\mLambda(\vx_i) \mJ(\vx_i) \rnd{\param - \param_*}, \mLambda(\vx_i)},
\end{equation}
where the original likelihood of the Bayesian NN model or GLM is replaced by a Gaussian likelihood.
The Gaussian likelihood matches its first two moments at $\param_*$ for likelihoods with dispersion parameter $\sigma^2=1$, since $\mLambda(\vx_i) = \myvar \sqr{ \mY_i }$ and the inverse link maps to the mean.
For the overdispersed Gaussian likelihood, this connection is not useful since the GLM is already a Bayesian linear regression model.
\end{theorem}

\begin{proof}
We apply the Laplace approximation to the model with the linearized neural network in \autoref{eq:glm_nn}.
We handle the log likelihood and log prior individually:
Since the log density of Gaussian is of second order, the prior remains unchanged.
We work with a single summand of the log likelihood and use the representative data pair $(\vx,\vy)$ and drop the dependency on $\vx$ in the notation to save space, i.e., for any operator $A(\vx)$ we write only $A$.
The residual vector $\vr(\vy, \vlogits(\vx;\param_*))$ is further abbreviated as $\vr$.
Therefore, the second-order Taylor approximation to the log likelihood around $\param_*$ required for the Laplace approximation can be written as
\begin{align*}
  \log p(\vy \given \linlogits^{\param_*}(\vx;\param))
  &\approx \log p(\vy \given \linlogits^{\param_*}(\vx;\param_*))
  + \vr \mJ \rnd{ \param - \param_* }
  - \frac12 \rnd{ \param - \param_* }^\top \mJ^\top \mLambda \mJ \rnd{ \param - \param_* }\\
  &=  \log p(\vy \given \vlogits)
  - (\vg^{-1}(\vlogits) - \vy) \rnd{\mJ \rnd{ \param - \param_* } }
  - \frac12 \rnd{\mJ \rnd{ \param - \param_* } }^\top \mLambda \rnd{\mJ \rnd{ \param - \param_* } }\\
  &= \log p(\vy \given \vlogits) + \frac12 \rnd{ \vg^{-1}(\vlogits) - \vy }^\top \mLambda^{-1} \rnd{ \vg^{-1}(\vlogits) - \vy } \\
  & \quad - \frac12 \rnd{ \vg^{-1}(\vlogits) + \mLambda \mJ \rnd{ \param - \param_* }  - \vy }^\top \mLambda^{-1} \rnd{  \vg^{-1}(\vlogits) + \mLambda \mJ \rnd{ \param - \param_* }  - \vy  },
\end{align*}
where we have expanded the residual $\vr$ and completed the square~(\autoref{cha:complete_square}).
Exponentiating the term and removing the parts independent of the parameter $\param$, we obtain the desired result.
For the Laplace approximation to the marginal likelihood, the remaining constants will be important~(cf. \autoref{sec:marglik_laplace}).
\end{proof}

The result makes two interesting things apparent:
first, the term $\mLambda \mJ$ is equal to the Jacobian of $\vg^{-1}(\logits(\vx;\param))$ evaluated at $\param_*$.
This follows from the GLM likelihoods since the derivative of the inverse link is $\mLambda$ and therefore we obtain the Jacobian by the chain-rule.
This indicates that Laplace-GGN implicitly performs a first-order Taylor approximation \emph{after} the inverse link function.
Second, when using the Laplace approximation to a linear model, we essentially turn the model into a solvable Bayesian linear regression model.
This Bayesian linear regression model matches the moments of the original likelihood at the expansion point $\param_*$.
In the case of maximum likelihood estimation, \citet{wedderburn1974quasi} has shown something similar:
using the generalized Gauss-Newton for generalized linear models requires iterative solutions of least-squares problems that are weighted by the response variances.

\section{From Generalized Gaussian Process to Gaussian Process Regression}%
\label{sec:gauss_laplace_gpr}

We have explained how the generalized linear model due to the GGN (\autoref{eq:glm_nn}) can be cast as a generalized Gaussian process model (\autoref{eq:glm_GP}).
Here, we will show that the Laplace approximation to this model yields a Gaussian process regression model.
Essentially, this allows us to do the Laplace-GGN in the function-space instead of the parameter-space.
Therefore, we can trade off computational costs between dimensionality $P$ and dataset size $N$~\cite{rasmussen2003gaussian}.
This holds for the posterior, posterior predictive, and the marginal likelihood.
Therefore, the Laplace-GGN is the first method that uses Gaussian process inference for finite width neural networks~\cite{jacot2018neural, lee2017deep}.
The relationship to exact Gaussian process regression presented in this section allows to use efficient standard routines developed for these models~\cite{rasmussen2003gaussian}.

We denote the Laplace-GGN posterior approximation in function-space by $q(\fgp)$.
In the following we show that it can be simply computed by exact Gaussian process inference.
\citet{rasmussen2003gaussian} derives the Laplace approximation to a Gaussian process model for classification likelihoods.
We denote the neural network linearized after the link function by
\begin{equation}
  \label{eq:lin_link}
  \vg^{-1}_\textrm{lin}(\vx_i;\param) = \vg^{-1}(\vf(\vx;\param_*)) + \mLambda(\vx) \mJ(\vx) (\param - \param_*).
\end{equation}
Using this notation, we can relate the Laplace approximation to the generalized Gaussian process:

\begin{theorem}
\label{thm:gp_laplace}
We define the mean and covariance function of a GP prior $\hat{\vf}_\textrm{GP} \sim \GP (\hat{\vm}(\vx), \hat{\vkappa}(\vx, \vx'))$
\begin{align}
  \hat{\vm}(\vx) &= \vg^{-1}_\textrm{lin} (\vx;\vmu_0) \quad \textrm{and} \quad \hat{\vkappa}(\vx, \vx') = \mLambda(\vx) \mJ(\vx)^\top \mSigma_0 \mJ(\vx') \mLambda(\vx'),
\end{align}
which gives rise to the prior $\hat{p}(\hat{\vf}_\textrm{GP})$.
Then, the Laplace approximation to the generalized Gaussian process model in \autoref{eq:glm_GP} at $\vf(\vx;\param_*)$ is equal to the posterior of the following Gaussian process regression model:
\begin{equation}
  \label{eq:lap_gp_regression}
  q(\fgp) = \hat{p}(\hat{\vf}_\textrm{GP} \given \data) \propto \hat{p}(\hat{\vf}_\textrm{GP}) \prod_{i=1}^N \gauss(\vy_i \given \hat{\vf}_\textrm{GP}, \mLambda(\vx_i)).
\end{equation}
This model complements the Bayesian linear regression model in \autoref{thm:laplace} and has the same marginal likelihood and posterior predictive~\cite{rasmussen2003gaussian}.
\end{theorem}

\begin{proof}
In correspondence with \autoref{sec:gauss_laplace_blr}, we conduct a Laplace approximation at $\vf_*:=\vf(\vx;\param_*)$.
Again, the prior is Gaussian and therefore remains unchanged.
For the likelihood of an arbitrary label $\vy$, we have
\begin{align*}
  \log p(\vy \given \fgp) &\approx  \log p(\vy \given \vf_*)  - (\vg^{-1} \rnd{ \vf_* } - \vy )^\top (\fgp- \vf_*) - \frac12 (\fgp- \vf_*)^\top \mLambda (\fgp- \vf_*) \\
  &= \log p(\vy \given \vf_*) + \frac12 (\vg^{-1}(\vf_*) - \vy)^\top \mLambda^{-1} (\vg^{-1}(\vf_*) - \vy)
  \\
  & \qquad \qquad \qquad - \frac12 (\vg^{-1}(\vf_*) + \mLambda \fgp - \mLambda \vf_*- \vy) \mLambda^{-1} (\vg^{-1}(\vf_*) + \mLambda \fgp- \mLambda \vf_* - \vy),
\end{align*}
where we expanded around $\vf_*$ and completed the square.
The first two summands of the last term are independent of $\fgp$ and the last term yields a Gaussian likelihood as in the proof of \autoref{thm:laplace}.
This allows to write a Gaussian likelihood.
Again, the remaining constants are necessary for the computation of the marginal likelihood.
\end{proof}

The kernel in this model resembles the \emph{neural tangent kernel} which arises when analyzing neural network training in function space~\cite{jacot2018neural}.
To see this, we can reparameterize the GP as $\mLambda(\vx) \hat{\vf}_\textrm{GP}$ and obtain the kernel $\hat{\vkappa}(\vx,\vx') = \mJ(\vx)^\top \mSigma_0 \mJ(\vx')$.
With a spherical prior covariance $\mSigma_0$, we recover the kernel of \citet{jacot2018neural}.
In contrast, we deal with finite width networks in the Bayesian setting and can recover similar properties using the practical Laplace-GGN method presented here.
On the practical side, Laplace-GGN in function-space enables inference corresponding to a full posterior covariance for neural networks with huge amounts of parameters but few data.
In particular, Laplace-GGN in function-space has the computational complexity $\mathcal{O}(N^3 K^3 + NPK)$ for inversion of the kernel and computation of the Jacobians.
Potentially, one can approximate this kernel by a low-rank structure, e.g., the Nyström approximation~\cite{rasmussen2003gaussian}, leading to novel approximate inference algorithms for neural networks.
In contrast, the complexity of the corresponding Bayesian linear regression inference is $\mathcal{O}(P^3+NPK)$.
In the following, we will derive the posterior, predictive, and marginal likelihood of both models.

\section{Computing the Laplace-GGN Posterior Approximation}%
\label{sec:posterior_of_gauss_laplace}

We compute the Laplace-GGN approximation to the posterior distribution of the neural network model.
To obtain the parameters of the Gaussian approximation $q(\param) = \gauss(\vmu, \mSigma)$, we make use of \autoref{thm:laplace}.
Bayesian linear regression models have a closed form solution~\cite{bishop2006pattern, murphy2012machine}.
The mean and covariance parameter of the Gaussian posterior $\hat{p} (\param \given \data)$ are given by
\begin{align}
  \mSigma &= \Big( \mSigma_0^{-1} + \sum_{i=1}^N \mJ(\vx_i) \mLambda(\vx_i) \mJ(\vx_i)^\top \Big)^{-1}, \\
  \vmu &= \mSigma \Big( \mSigma_0^{-1} \vmu_0 + \sum_{i=1}^N \mJ(\vx_i)^\top (\vy_i - \vg^{-1}\rnd{\vlogits(\vx_i)}) + \mJ(\vx_i)^\top \mLambda(\vx_i) \mJ(\vx_i) \param_* \Big).
\end{align}
The covariance is simply the inverse of the Hessian of the negative log joint distribution.
At first, the mean seems more involved.
However, assuming a local minimum, i.e. $\sum_{i=1}^N \mJ_*(\vx_i) (\vy_i - \vg^{-1}_*(\vx_i)) - \mSigma_0^{-1}(\param_* - \vmu_0) = \vzero$, we obtain have $\vmu = \param_*$.
In this case, we recover a Laplace approximation that is constructed at the MAP, which we did not need to assume.

Alternatively, we can use the posterior of the generalized Gaussian process and infer in function space.
Recall the kernel from the beginning of this chapter in \autoref{eq:lap_gp_kernel}.
Due to the multiple outputs, our kernel maps to a matrix.
The kernel $\mK \in \R^{NK \times NK}$ consists of $N$ $(K\times K)$ \emph{submatrices} along both dimensions.
The \emph{submatrix} at the $i$-th position along the first and $j$-th position along the second axis is given by $\vkappa(\vx_i, \vx_j) \in \R^{K \times K}$.
We further have the kernel $\mK_{**} = \kappa(\vx_*, \vx_*)$ for data point $\vx_*$ and the joint kernel with the training data $\mK_{*n} \in \R^{K \times NK}$.
The block-diagonal matrix $\mW \in \R^{NK\times NK}$ is defined by $N$ $(K \times K)$ blocks where the $i$-th block is the negative log likelihood Hessian $\mLambda(\vx_i)$.
Assuming a stationary MAP estimate, we have the following distribution on $\fgp$ due to the Laplace approximation:
\begin{equation}
  \label{eq:laplace_gp_posterior}
  \fgp \given \data, x_* \sim \gauss \rnd{ \vf(\vx_*; \param_*), \mK_{**} -  \mK_{*n} \rnd{ \mK + \mW^{-1} }^{-1} \mK_{*n}^\top }.
\end{equation}

\section{Posterior Predictive Distributions}%
\label{sec:predictive_laplace}

This section manifests how the generalized Gauss-Newton method and approximate inference change our underlying inference model.
Due to the two steps depicted in Figure~\ref{eq:predictive_lap_blr}, we obtain three different posterior predictive models.
The original model is a probabilistic neural network model.
Due to the GGN, we obtain a generalized linear or Gaussian process model.
Finally, the Laplace approximation corresponds to exact inference in a Bayesian linear or Gaussian process regression model.
Subsequently, we specify all three (approximate) posterior predictive distributions.
We predict the label $\vy_*$ of a new data point $\vx_*$ and therefore need to compute the predictive distribution $p(\vy_* \given \data, \vx_*)$.
All models have the same (approximate) posterior distribution but the posterior predictive is naturally different.

\textbf{NN sampling:} to make predictions with the neural network, we need to approximate the posterior predictive integral by sampling.
With $S$ Monte Carlo samples $\param_1, \ldots, \param_S \sim q(\param)$, we have
\begin{equation}
  \label{eq:predictive_lap_nn}
  p(\vy_* \given \data, \vx_*) \approx \int p(\vy_* \given \logits(\vx_*;\param)) q(\param) d\param \approx \frac{1}{S} \sum_{i=1}^S p(\vy_* \given \logits(\vx_*;\param_s)).
\end{equation}

\textbf{GLM sampling:} to predict using the generalized linear or Gaussian process model, we also need samples to approximate the posterior predictive integral unless we have a Gaussian likelihood.
Again using $S$ samples, we have
\begin{equation}
  \label{eq:predictive_lap_glm}
  \tilde{p}(\vy_* \given \data, \vx_*) \approx \int p(\vy_* \given \linlogits(\vx_*;\param)) q(\param) d\param \approx \frac{1}S \sum_{i=1}^S p(\vy_* \given \linlogits^{\param_*}(\vx_*;\param_s)).
\end{equation}

\textbf{BLR predictive:} the Laplace-GGN leads to exact inference in a Bayesian linear regression model.
The posterior predictive is available in a closed form but the distribution does not match the original likelihood.
We have
\begin{equation}
  \label{eq:predictive_lap_blr}
\begin{split}
  \hat{p}(\vy_* \given \data, \vx_*) &= \int \gauss(\vy_* \given \linlink(\vx_*;\param), \mLambda(\vx_*) ) q(\param) d\param \\
  &= \gauss \rnd{ \vg^{-1}_\textrm{lin}(\vx_*;\vmu), \mLambda(\vx_*) \mJ(\vx_*) \mSigma \mJ(\vx_*)^\top \mLambda(\vx_*) + \mLambda(\vx_*) }.
\end{split}
\end{equation}
At the MAP, the mean simply is the MAP prediction of the neural network with additional uncertainty.
In the case of a Gaussian likelihood, the last step is not needed.
For the other likelihoods, the Hessian $\mLambda(\vx_*)$ corresponds to the variance of the GLM response variable.
For the BLR and GLM sampling predictive, we can equivalently use the Gaussian process view under use of the posterior Gaussian process introduced in the previous section.

Only the first posterior predictive approximation is used in practice~\cite{blundell2015weight, khan2018fast, ritter2018scalable, zhang2018noisy}.
However for a Bayesian neural network with scalar Gaussian likelihood, \citet{foong2019between} have recently shown empirically that the third version often works better. The development of this chapter theoretically underline this result and extends it to other likelihoods.
Applying the generalized Gauss-Newton approximation changes our underlying inference model to a generalized linear model.
Therefore, this is the model we should predict with when using the Laplace-GGN.
The closed form Bayesian linear regression version is the second choice since it maintains linearity.
However, the likelihood of the original model is replace by a Gaussian.
In \autoref{sec:experiments}, we support this hypothesis experimentally and show that the first predictive can fail spectacularly.

\section{Marginal Likelihood Approximation}%
\label{sec:marglik_laplace}

In this section, we derive the Laplace-GGN marginal likelihood approximation.
We can derive it from the Laplace approximation to the marginal likelihood of the generalized linear or Gaussian process model.
Using the exact marginal likelihoods as given by the BLR and GP regression models, we only have to derive a \emph{correction} term.
In the proofs of \autoref{thm:laplace} and~\ref{thm:gp_laplace}, the second-order approximation to the likelihood is the same and we can use it to derive the following result:

\begin{theorem}
  Let $\hat{p}(\data)$ be the marginal likelihood of the Bayesian linear or Gaussian process regression model (\autoref{thm:laplace} and \ref{thm:gp_laplace}).
  Then, the Laplace approximation to the marginal likelihood of the generalized linear model denoted by $\tilde{p}(\data)$ is given by
  \begin{equation}
    \label{eq:marglik_laplace}
    \log \tilde{p}(\data) \approx \log \hat{p}(\data) +  \sum_{i=1}^N \Big[ \log p(\vy_i \given \vf(\vx_i;\param_*)) - \log \gauss (\vy_i \given \vg^{-1}(\vf(\vx_i;\param_*)), \mLambda(\vx_i))\Big].
  \end{equation}
  For the Gaussian likelihood, this leads to $\hat{p}(\data)$ while other likelihoods maintain a correction term.
\end{theorem}

\begin{proof}
  The Laplace approximation to the generalized linear and GP model marginal likelihood arises from a second-order approximation.
  We can simply start off from the proofs of \autoref{thm:laplace} and~\ref{thm:gp_laplace}.
  For the prior, a second order approximation of the log density is exact since it is a Gaussian.
  For the likelihood, we obtain the following term for a single input-output pair:
  \begin{align*}
    \log p(\vy \given \vf) &\approx \log p(\vy \given \vlogits_*) + \frac12 \rnd{ \vg^{-1}(\vlogits_*) - \vy }^\top \mLambda^{-1} \rnd{ \vg^{-1}(\vlogits_*) - \vy } \\
    & \quad - \frac12 \rnd{ \vg^{-1}(\vlogits_*) + \mLambda \mJ \rnd{ \param - \param_* }  - \vy }^\top \mLambda^{-1} \rnd{  \vg^{-1}(\vlogits_*) + \mLambda \mJ \rnd{ \param - \param_* }  - \vy  }.
  \end{align*}
  Let us add and subtract the term $\frac12 \log \rnd{ (2\pi)^k \det{\mLambda} }$ to obtain a Gaussian log likelihood in the second term.
  The second term becomes $-\log \gauss (\vy \given \vg^{-1}(\vf(\vx;\param_*)), \mLambda(\vx))$.
  The last term is also properly normalized to a Gaussian.
  The first and second term are independent of the parameter $\param$ and give us the correction terms.
  The remaining term applied to the entire dataset and combined with the prior gives us the BLR or GP regression models defined earlier.
  We denote the marginal likelihood of these models as $\hat{p}(\data)$.
\end{proof}

Having established this connection, we are only left with the computation of $\hat{p}(\data)$.
For the term $\hat{p}(\data)$, we have again two ways to compute it due to weight and function space equivalence.
Estimating $\hat{p}(\data)$ using the Bayesian linear regression model of \autoref{thm:laplace} is widely known and there exist numerically robust implementations for it~\cite{bishop2006pattern, murphy2012machine}.
We have
\begin{equation}
  \label{eq:marglik_laplace_blr}
  \log \hat{p}(\data) =  \sum_{i=1}^N \log \gauss (\vy_i \given \vg_{\vmu}^{-1}(\vx_i)) - \frac12 \log \frac{\det \mSigma_0}{\det \mSigma} - \frac12 \rnd{ \vmu - \vmu_0 }^\top \mSigma_0^{-1} \rnd{ \vmu - \vmu_0 }.
\end{equation}
Alternatively, we can use the Gaussian process variant.
The key difference lies in the computational complexity shifted into the number of samples $N$ as opposed to the number of parameters $P$.
In contrast to the kernel for the generalized GP in \autoref{sec:predictive_laplace}, we have the kernel of the GP regression model:
$\hat{\mK} \in \R^{NK\times NK}$ consists of $N^2$ $(K \times K)$ submatrices.
The submatrix at $i$-th position along the first dimension and $j$-th position along the second axis is given by $\hat{\vkappa}(\vx_i, \vx_j)$ defined in \autoref{thm:gp_laplace}.
$\mW \in \R^{NK \times NK}$ is a block-diagonal matrix of the $N$ Hessians $\mLambda(\vx_i)$ for $i \in [N]$.
Lastly, the mean function $\vm \in R^{NK}$ is a concatenation of the individual mean functions of the $N$ data points and $\vy$ a concatenation of the corresponding labels.
Then, we have for the marginal likelihood of the GP regression model
\begin{equation}
  \label{eq:marglik_laplace_gp}
  \log \hat{p} (\data) = - \frac12 \rnd{\vm - \vy}^\top \rnd{ \mK + \mLambda }^{-1}\rnd{\vm - \vy} - \frac12 \det \sqr{ \mK + \mLambda } - \frac{NK}2 \log 2\pi.
\end{equation}
With a single output model, we recover the form derived by \citet[Sec. 2]{rasmussen2003gaussian}.

\section{The Prior as Regularizer and Distribution}%
\label{sec:laplace_prior}

In recent discussions, the prior in Bayesian neural network models and its obscure meaning have been criticized~\cite{gelada2020bnncrit, wilson2019bayesian}.
Here, we discuss the role of the prior in a Bayesian deep learning algorithm.
In particular, we discuss the role of the prior in an approximate as opposed to exact posterior.

This chapter shows that, in the case of the Laplace-GGN approximation, the prior acts in form of a distribution in a GLM and as a regularizer in the MAP objective.
The regularizer impacts the learned feature map, i.e., Jacobian, as well as the linearization point that both give rise to the GLM that we infer.
The GLM with fixed feature map and a Gaussian prior is simple to infer and has a unimodal posterior.
While it might be complicated to pose a prior on neural network parameters, the Laplace-GGN approximation simplifies the model to a GLM where a simple Gaussian prior is very common and justified~\cite{murphy2012machine}.
Therefore, it is not straightforward to criticize the priors used in Bayesian deep learning since the prior impacts an approximate and not exact posterior.
For the Laplace-GGN, we rather have to ask if the MAP estimation and linearization at the MAP is reasonable.
The error of the linearization can be quantified while MAP estimation forms the foundation of deep learning itself.
If we find both to be reasonable, inference in a GLM using the Laplace approximation is a minor remaining problem.



\chapter{The Impact of the Generalized Gauss-Newton in Variational Inference}
\label{sec:vi}
The Gaussian variational approximation is the direct competitor of the Laplace approximation:
it is equally scalable and uses the same approximating family.
The difference is that optimization and approximate inference steps are not sequential but combined into one \emph{variational inference} algorithm that maximizes the ELBO (\autoref{eq:elbo}).
Nonetheless, we can disentangle the respective influence of the generalized Gauss-Newton method and approximate inference for each step of such an algorithm.
This understanding again suggests different posterior predictive distributions, updates in function-space due to a Gaussian process formulation, and leads to the identification of a new algorithm.
To maximize the ELBO, we make use of \emph{natural gradient variational inference} (NGVI).
NGVI uses the information geometry to improve convergence~\cite{amari1998natural} and is responsible for recent successes in the field of Bayesian deep learning~\cite{khan2018fast, osawa2019practical, zhang2018noisy}.
Most algorithms further rely on the generalized Gauss-Newton approximation.
\citet{zhang2018noisy} use a Kronecker-factored approximation and \citet{khan2018fast} use a diagonal approximation to the GGN.
Here, we analyze the case of the full GGN approximation.

First, we specify the parameter updates of a natural gradient variational inference method for the GVA~\cite{khan2017conjugate,khan2018fast}.
A short derivation of NGVI for the GVA can be found in \autoref{ch:ngvi}.
We introduce three algorithms derived from the NGVI update to a Gaussian posterior approximation.
Two of these algorithms, \emph{variational online generalized Gauss-Newton} (VOGGN) and \emph{online generalized Gauss-Newton} (OGGN), have been introduced before~\cite{khan2019approximate}.
Additionally, we derive a new algorithm, the \emph{linearized Gaussian variational approximation} (LGVA) algorithm.
It is derived as a compromise between VOGGN and OGGN.
OGGN is similar to an iterative Laplace approximation as it does not sample.
The difference between VOGGN and LGVA lies in the order of operations:
VOGGN samples first and then uses the GGN while LGVA applies the GGN and then samples.
An illustration of the order of operations of the three algorithms is given in \autoref{tab:ngvi_schema}.
Finally, we show how the Gaussian variational approximation at every step of these NGVI algorithms can be cast as a Bayesian linear or Gaussian process regression problem.

\begin{table}[t]
  \centering
  \begin{tabular}{c l l}
  \small{Algorithm} & \small{Step $1$} & \small{Step $2$}  \\
  \toprule
  \small{VOGGN} & \small{sample} $\param_1,.., \param_S \sim \gauss(\vmu_t, \mSigma_t)$ & \small{$S$ lin. networks} $\vf_\textrm{lin}^{\param_1}(\vx;\param),\ldots,\vf_\textrm{lin}^{\param_S}(\vx;\param)$ \\
  \small{LGVA} & \small{$1$ lin. network} $\vf_\textrm{lin}^{\vmu_t}(\vx;\param)$ & \small{sample} $\param_1,.., \param_S \sim \gauss(\vmu_t, \mSigma_t)$ \\
  \small{OGGN} & \small{$1$ lin. network} $\vf_\textrm{lin}^{\vmu_t}(\vx;\param)$ & \small{use the mean} $\vmu_t$  \\
  \end{tabular}
  \caption{Illustration of three variational generalized Gauss-Newton algorithms. The order of approximating the expectation by samples and applying the linearization of the GGN yields different algorithms.
  Sampling is part of the variational inference algorithm while linearization comes from the GGN.
  Only VOGGN samples many neural networks.
  OGGN is a crude version of LGVA and VOGGN since it uses only the mean instead of sampling.}
  \label{tab:ngvi_schema}
\end{table}

\section{Gaussian Natural Gradient Variational Inference}%
\label{sec:ngvi}

We denote by $q_t(\param) = \gauss (\vmu_t, \mSigma_t)$ the Gaussian variational approximation to the posterior at iteration $t$.
In this chapter, we will work with the natural parameters.
For the Gaussian distribution, the natural parameters at iteration $t$ are given by
\begin{equation}
  \label{eq:gaussian_natural}
  \natparam_t = \mSigma^{-1}_t \vmu_t \quad \textrm{and} \quad \natparas_t = - \frac12 \mSigma^{-1}_t.
\end{equation}
It is mathematically convenient to work with this parameterization for NGVI.
NGVI updates the first and second natural parameter of the Gaussian as
\begin{align}
  \mSigma_{t+1}^{-1} \vmu_{t+1} &= (1 - \gamma)\mSigma_t^{-1} \vmu_t + \gamma \mSigma_0^{-1}\vmu_0 + \gamma \sqr{  \nabla_\vmu \myexpect \sqr{ \log p(\data \given \param) } - 2 \nabla_\mSigma \myexpect \sqr{ \log p(\data \given \param) } \vmu_t  } \label{eq:first_nat_update},\\
  - \frac12 \mSigma_{t+1}^{-1} &= (1 - \gamma) \sqr{ - \frac12 \mSigma_{t}^{-1} } + \gamma \sqr{ - \frac12 \mSigma_{0}^{-1}  } + \gamma \nabla_\mSigma \myexpect \sqr{ \log p(\data \given \param) }, \label{eq:secnd_nat_update}
\end{align}
where the expectation is taken over the posterior approximation $q_t$ at iteration $t$.
The update tells us that we combine the current posterior approximation $q_t$ with the prior using a convex combination (usually $\gamma \leq 1$).
The data dependency is only due to the gradients with respect to mean and covariance of the expected log likelihood terms.
Due to the linearity of expectation, the expected log likelihood can be written as
\begin{equation}
\label{eq:lin_exp}
  \myexpect \sqr{ \log p(\data \given \param) } = \myexpect \sqr{ \sum_{i=1}^N \log p(\vy_i \given \vf(\vx_i;\param)) } = \sum_{i=1}^N \myexpect \sqr{ \log p(\vy_i \given \vf(\vx_i;\param))  },
\end{equation}
which allows us to restrict ourselves to a single representative data pair $(\vx, \vy)$.\footnote{In a practical scenario, one can use doubly-stochastic variational inference by sampling subsets of data to obtain an unbiased estimate\cite{hoffman2013stochastic}.}
Therefore, the problem reduces to estimation of the derivatives $\nabla_\vmu \myexpect \sqr{ \log p(\vy \given \vf(\vx;\param)) }$ and $\nabla_\mSigma \myexpect \sqr{ \log p(\vy \given \vf(\vx;\param)) }$.
To enable the computation of these gradients, we use the identities made popular by \citet{opper2009variational} that we introduced in \autoref{sec:ABI} (see \autoref{eq:bonnet_mean} and~\ref{eq:bonnet_var}).
This allows to express the gradients of the expectation as the expectation of gradients for individual samples from $q_t$.
To derive different Gaussian NGVI algorithms, different approximations to these derivatives have been proposed based on the GGN~\cite{khan2018fast, khan2019approximate, zhang2018noisy}.
After introducing the VOGGN algorithm, we will add one more variant to this family of Gaussian NGVI algorithms.
Recall that $\vf_\textrm{lin}^{\param_*}(\vx;\param)$ denotes the neural network function linearized at $\param_*$ due to the GGN.
Further, we denote $S$ Monte Carlo samples from the approximating distribution $q_t(\param)=\gauss(\vmu_t, \mSigma_t)$ by $\param_1,\ldots,\param_S$.

In all the following derivations, we first apply the identity of \citet{opper2009variational}.
This allows us to estimate the derivative with respect to the mean and covariance by sampling individual gradients.
Recall the equalities from \autoref{sec:ABI}
\begin{align}
  \nabla_\vmu \myexpect_{\param_s \sim \gauss(\vmu_t, \mSigma_t)} \sqr{ \log p(\vy \given \vf(\vx;\param_s)) }
  &= \myexpect_{\param_s \sim \gauss(\vmu_t, \mSigma_t)} \sqr{ \nabla_\param \log p(\vy \given \vf(\vx;\param_s))} \\
  \nabla_\mSigma \myexpect_{\param_s \sim \gauss(\vmu_t, \mSigma_t)} \sqr{\log p(\vy \given \vf(\vx;\param_s))}
  &= \frac12 \myexpect_{\param_s \sim \gauss(\vmu_t, \mSigma_t)} \sqr{ \nabla_{\param \param}^2 \log p(\vy \given \vf(\vx;\param_s))}.
\end{align}
All following algorithms vary only in their approximation to these expectations and the log likelihood or its gradient and Hessian.
We now show three different variants.

\section{Variational Online Generalized Gauss-Newton}%
\label{sec:vogn}

The variational online generalized Gauss-Newton algorithm uses the GGN \emph{after} sampling parameters from the approximating distribution.
For the first and second derivative, we have
\begin{align}
\begin{split}
   \myexpect_{\param_s \sim \gauss(\vmu_t, \mSigma_t)} \sqr{ \nabla_\param \log p(\vy \given \vf(\vx;\param_s))}
  &\approx \frac1S \sum_{i=1}^S \mJ(\vx;\param_s)^\top \vr(\vy, \vf(\vx;\param_s)),
\end{split}\\
\begin{split}
\label{eq:vogn_sigma_update}
  \myexpect_{\param_s \sim \gauss(\vmu_t, \mSigma_t)} \sqr{ \nabla_{\param \param}^2 \log p(\vy \given \vf(\vx;\param_s))}
  &\approx \frac1{2} \sum_{i=1}^S \nabla_{\param \param}^2 p (\vy \given \vf(\vx;\param_s)) \\
  &\approx \frac1{2} \sum_{i=1}^S \nabla_{\param \param}^2 p(\vy \given \vf_\textrm{lin}^{\param_s}(\vx;\param_s))\\
  &= -\frac1{2} \sum_{i=1}^S \mJ(\vx;\param_s)^\top \mLambda(\vf(\vx;\param_s)) \mJ(\vx;\param_s).
\end{split}
\end{align}
For the first derivative, simply approximate the expected gradient using $S$ samples.
For the Hessian, we first sample $S$ neural network models and then linearize these models individually at the sampled parameters $\param_s$.
We consider this expansion point as a constant and therefore can compute the Hessian of the linearized neural network log likelihood with respect to individual samples $\param_s$.
This is like simultaneously sampling the linearization point and parameter.
Similar algorithms proposed before, have used either a diagonal, low-rank, or Kronecker factored approximation to the Hessian~\cite{blundell2015weight, khan2018fast, mishkin2018slang, zhang2018noisy}.
Next, we introduce a new algorithm that does not apply the GGN per sample but rather before sampling.

\section{Linearized Gaussian Variational Inference}%
\label{sec:lin_gva}

The linearized GVA applies the generalized Gauss-Newton approximation \emph{before} sampling.
That means, we linearize the neural network at some point $\param_*$ and then compute the gradients.
Here, we choose to linearize at $\vmu_t$.
Therefore, we have
\begin{align}
\begin{split}
\label{eq:lgva_update_mu}
   \myexpect_{\param_s \sim \gauss(\vmu_t, \mSigma_t)} \sqr{ \nabla_\param \log p(\vy \given \vf(\vx;\param_s))}
  &\approx \myexpect_{\param_s \sim \gauss(\vmu_t, \mSigma_t)} \sqr{ \nabla_\param \log p(\vy \given \vf_\textrm{lin}^{\vmu_t}(\vx;\param_s))} \\
  &= \frac1S \sum_{i=1}^S \mJ(\vx;\vmu_t)^\top \vr(\vy,\vf_\textrm{lin}^{\vmu_t}(\vx;\param_s)),
\end{split}\\
\begin{split}
\label{eq:lgva_update_sigma}
  \myexpect_{\param_s \sim \gauss(\vmu_t, \mSigma_t)} \sqr{ \nabla_{\param \param}^2 \log p(\vy \given \vf(\vx;\param_s))}
  &\approx \myexpect_{\param_s \sim \gauss(\vmu_t, \mSigma_t)} \sqr{ \nabla_{\param \param}^2 \log p(\vy \given \vf_\textrm{lin}^{\vmu_t}(\vx;\param_s))} \\
  &= \myexpect_{\param_s \sim \gauss(\vmu_t, \mSigma_t)} \big[ -\mJ(\vx;\vmu_t)^\top \mLambda(\vf_\textrm{lin}^{\vmu_t}(\vx;\param_s)) \mJ(\vx;\vmu_t) \big] \\
  &= -\frac{1}{S} \sum_{i=1}^S \mJ(\vx;\vmu_t)^\top \mLambda(\vf_\textrm{lin}^{\vmu_t}(\vx;\param_s)) \mJ(\vx;\vmu_t).
\end{split}
\end{align}
In contrast to VOGGN, we only sample in the first order of the neural network.
LGVA has two potential advantages over VOGGN:
we need to compute only one Jacobian no matter how many samples we take and linearization might stabilize the training.
LGVA can be seen the variational twin of the Laplace-GGN since it constructs a GLM in each step and not only at the MAP.
In this GLM, we take a step of natural gradient variational inference.
Therefore, we should predict with this model using the GLM sampling method.
The reason is easy to see:
in the above derivation, we only work with the linearized neural network.

\section{Deterministic Variational Online Gauss-Newton}%
\label{sec:det_vogn}

The last algorithm we introduce is called online generalized Gauss-Newton (OGGN).
It is motivated as a deep learning optimizer derived from natural gradient variational inference~\cite{khan2019approximate}.
Instead of sampling to compute expectations, we simply take the current mean $\vmu_t$.
Therefore, this method is related to the Laplace approximation.
We have the derivatives
\begin{align}
\begin{split}
   \myexpect_{\param_s \sim \gauss(\vmu_t, \mSigma_t)} \sqr{ \nabla_\param \log p(\vy \given \vf(\vx;\param_s))}
  &\approx \nabla_\vmu \log p(\vy \given \vf(\vx;\vmu_t))\\
  &= \mJ(\vx;\vmu_t)^\top \vr(\vy,\vf(\vx;\vmu_t),
\end{split}\\
\begin{split}
  \myexpect_{\param_s \sim \gauss(\vmu_t, \mSigma_t)} \sqr{ \nabla_{\param \param}^2 \log p(\vy \given \vf(\vx;\param_s))}
  &\approx \nabla_{\vmu \vmu}^2 \log p(\vy \given \vf(\vx;\vmu_t)) \\
  &\approx -\mJ(\vx;\vmu_t)^\top \mLambda(\vf(\vx;\vmu_t))\mJ(\vx;\vmu_t),
\end{split}
\end{align}
where we approximate the expectation using the mean.
For the second derivative, we use the GGN to get the last line.
Note also that we can obtain above algorithm starting from LGVA and using the mean $\vmu_t$ instead of sampling parameters $\param_s$.

\section{Variational GGN Iterations as Exact Inference}%
\label{sec:vogn_blr}

In line with the proofs for the Laplace-GGN approximation, we will show that all above NGVI algorithms solve local Bayesian linear regression models.
Again, we can interpret these models in the function-space and characterize them as Gaussian processes.
This analysis also explains why VOGGN is the most powerful algorithm and can be expected to predict well when we sample from the neural network.
LGVA and OGGN are therefore expected to require linearization to predict accurately.

The natural parameter updates of the Gaussian variational approximation in \autoref{eq:first_nat_update} and~\ref{eq:secnd_nat_update} can simply be written into the Gaussian posterior approximation $q_{t+1}$ by multiplying with the Gaussian sufficient statistics (see \autoref{ch:ngvi}).
The first step is to combine the prior and the posterior approximation at step $t$ to obtain an intermediary prior.
We take the terms independent of the data from the natural parameter updates and define the natural parameters of the Gaussian $p_t(\param)=\gauss(\vm, \mS)$ as $\eta^{(1)} = (1-\gamma) \mSigma_t^{-1} \vmu_t + \gamma \mSigma_0^{-1} \vmu_0$ and $\eta^{(2)} = - \frac12 \sqr{ (1-\gamma) \mSigma_t^{-1} + \gamma \mSigma_0^{-1} }$.
Resolving the data-dependent term requires more steps and is shown in the proof of the following theorem.

\begin{table}[t]
  \centering
  \begin{tabular}{c | l l l}
  Algorithm & samples & $\hat{\vf}_s(\vx)$  & $\hat{\mJ}_s(\vx)$ \\
  \toprule
  VOGGN & $S$ & $\vf(\vx;\param_s)$  & $\mJ(\vx;\param_s)$ \\
  LGVA & $S$ & $\vf^{\vmu_t}_\textrm{lin}(\vx;\param_s)$ & $\mJ(\vx;\vmu_t)$ \\
  OGGN & $1$ & $\vf^{\vmu_t}_\textrm{lin}(\vx;\vmu_t)=\vf(\vx;\vmu_t)$ & $\mJ(\vx;\vmu_t)$ \\
  \end{tabular}
  \caption{Values of the parameters in \autoref{thm:ngvi} for the three algorithms. $\param_s$ is a sample from the posterior approximation $q_t$ at iteration $t$ and $\vmu_t$ its mean. Only VOGGN samples $S$ neural networks and obtains individual Jacobians.}
  \label{tab:ngvi_theorem}
\end{table}

\begin{theorem}
\label{thm:ngvi}
  The VOGGN, LGVA, and OGGN algorithms perform exact Bayesian linear regression in each update.
  In the most general case, we can characterize the updated posterior approximation $q_{t+1}(\param)$ as
  \begin{align}
    q_{t+1}(\param) \propto p_t(\param) \prod_{i=1}^N \prod_{i=1}^S \gauss \rnd{ \vy_i \Big| \vg^{-1}(\hat{\vf}_s(\vx_i)) + \mLambda(\hat{\vf}_s(\vx_i))\hat{\mJ}_s(\vx_i)(\param - \vmu_t), \frac{S}\gamma \mLambda(\hat{\vf}_s(\vx_i)) },
  \end{align}
  where $S$ is the number of Monte Carlo samples, $\gamma$ the step size and the function and Jacobian values $\hat{\vf}_s$ and $\hat{\mJ}_s$ depend on the algorithm.
  For the particular values, see \autoref{tab:ngvi_theorem}.
  The key difference to \autoref{thm:laplace} lies in the fact that we sample $\param_s$ from $q_t(\param)$ as opposed to taking the mean.
  For the overdispersed Gaussian likelihood, this also holds but needs to be written slightly differently since $\mLambda(\hat{\vf}_s(\vx))$ does not correspond to the noise variance (see end of proof below).
\end{theorem}

\begin{proof}
The prior $p_t(\param)$ arises from the natural parameter update.
Additionally, we plug the data-dependent terms into the Gaussian natural parameterization (cf. \autoref{sec:GNMs} and \autoref{ch:ngvi}).
In particular, we need to plug into $e^{\gamma \langle T(\param), \widetilde{\veta}\rangle}$ where $\widetilde{\veta}$ denotes the data-dependent natural parameter summands of \autoref{eq:first_nat_update} and~\ref{eq:secnd_nat_update} as follows
\begin{equation*}
  e^{\gamma \langle T(\param), \widetilde{\veta}\rangle} = \exp \crl{ \gamma \langle \param, \nabla_\vmu \myexpect \sqr{ \log p(\data \given \param) } - 2 \nabla_\mSigma \myexpect \sqr{ \log p(\data \given \param) } \vmu_t  \rangle + \gamma \langle \param \param^\top, \nabla_\mSigma \myexpect \sqr{ \log p(\data \given \param) } \rangle  }.
\end{equation*}
Next, we can use the linearity of expectation and write $\myexpect \sqr{ \log p(\data \given \param) }$ as a sum over the $N$ data points.
Since the inner product is linear, we can pull the sum outside of the exponent and get a product over $N$ data points instead:
\begin{equation*}
  \prod_{i=1}^N \exp \crl{ \gamma \langle \param, \nabla_\vmu \myexpect \sqr{ \log p(\vy_i \given \param) } - 2\nabla_\mSigma \myexpect \sqr{ \log p(\vy_i \given \param) \vmu_t} \rangle + \gamma \langle \param \param^\top, \nabla_\mSigma \myexpect \sqr{ \log p(\vy_i \given \param) }  \rangle },
\end{equation*}
where we have abbreviated $\log p(\vy \given \vf(\vx;\param))$ as $\log p(\vy \given \param)$.
Taking $S$ samples $\param_1,\ldots,\param_S \sim q_t(\param)$ leads to a sum over $S$ divided by $S$.
The sum can again be pulled outside to obtain a product over these samples and we pull the gradient inside the expectation to obtain
\begin{equation*}
  \prod_{i=1}^N \prod_{s=1}^S \exp \crl{ \frac{\gamma}{S} \langle \param, \nabla_\param \log p(\vy_i \given \param_s)  - \nabla_{\param \param}^2 \log p(\vy_i \given \param_s) \vmu_t \rangle + \frac{\gamma}{2S} \langle \param \param^\top, \nabla_{\param \param}^2  \log p(\vy_i \given \param_s) \rangle }.
\end{equation*}
We continue with the exponent for a single data and MC sample.
We use $\hat{f}_s(\vx)$ for the function and $\hat{\mJ}_s(\vx)$ for the Jacobian for some data point $(\vx, \vy)$ and parameter sample $\param_s$.
For brevity, we write $\hat{\mLambda}_s:=\mLambda(\hat{\vf}_s(\vx))$.
Then, we have for a single exponent
\begin{align*}
  &\frac{\gamma}{S} \langle \param, \mJ_s(\vx)^\top \vr(\vy, \hat{\vf}_s(\vx)) + \mJ_s(\vx)^\top \hat{\mLambda}_s\mJ_s(\vx) \vmu_t \rangle - \frac{\gamma}{2S} \langle \param \param^\top, \mJ_s(\vx)^\top \hat{\mLambda}_s \mJ_s(\vx) \rangle   \\
  = &  \frac{\gamma}S \param^\top \mJ_s(\vx)^\top \rnd{\vy - \vg^{-1}(\vf_s(\vx)) + \hat{\mLambda}_s \mJ_s(\vx) \vmu_t } - \frac{\gamma}{2S} \param^\top \mJ_s(\vx)^\top \hat{\mLambda}_s \mJ_s(\vx) \param  \\
  = &- \frac12 \rnd{\vg^{-1}(\hat{\vf}_s(\vx)) + \hat{\mLambda}_s \mJ_s(\vx) \rnd{ \param - \vmu_t } -  \vy } \rnd{\frac{S}\gamma \hat{\mLambda}_s}^{-1} \rnd{\vg^{-1}(\hat{\vf}_s(\vx)) + \hat{\mLambda}_s\mJ_s(\vx) \rnd{ \param - \vmu_t } -   \vy  } \\
  &+ \frac12 \rnd{ \vy - \vg^{-1}(\hat{\vf}_s(\vx)) } ^\top \rnd{ \frac{S}{\gamma}\hat{\mLambda}_s}^{-1} \rnd{  \vy - \vg^{-1}(\hat{\vf}_s(\vx))  }
\end{align*}
where we first used the inner product properties and then completed the square.
The first term in the individual exponent yields a Gaussian density with the desired structure and therefore concludes the proof.
Note that for OGGN, we do not sample so the proof is simpler but follows the same steps.
For the overdispersed Gaussian likelihood, we can set $\hat{\mLambda}_s=\mI_K$ above and then divide all terms by the dispersion parameter, i.e., variance.
That is only possible because both residual and Hessian are scaled by $\sigma^{-2}$ (see. \autoref{tab:glms}).
Then, we maintain a Gaussian likelihood.
\end{proof}

In comparison to \autoref{thm:laplace} obtained for the Laplace-GGN approximation, this theorem characterizes the steps of an approximate inference algorithm as opposed to the stationary point.
Therefore, we additionally have the step size $\gamma$ in our model.
The deterministic OGGN algorithm ($S=1$) is similar to the Laplace-GGN approximation.
This is apparent if we set the step size $\gamma$ of OGGN to $1$ at a stationary point:
The linear regression model in \autoref{thm:laplace} and \autoref{thm:ngvi} become equivalent.
OGGN has the advantage that it is an online algorithm and provides a posterior approximation in every step and not only at a MAP estimate.
All derived quantities in \autoref{sec:laplace} can also be applied to the OGGN posterior approximation: in particular, we should predict using GLM sampling in \autoref{eq:predictive_lap_glm}.

The Bayesian linear regression model corresponding to VOGGN and LGVA varies significantly due to the $S$ samples.
Each sample \emph{augments} the model with $N$ new predictive models that always predict the same $N$ targets $\vy$ from observations $\vx$.
In \autoref{sec:ABI}, we have characterized the stationarity of the Gaussian variational approximation, which depends on an expectation.
Here, we have taken samples to approximate this expectation and observe that this indeed leads to a more \emph{global} characterization due to an augmented linear regression model.
Therefore, comparing OGGN with VOGGN and LGVA is similar to the relation between Laplace and Gaussian variational approximation presented in \autoref{sec:ABI} but for iterations instead of stationarity.
Following the developments of \autoref{sec:laplace}, we can equivalently turn above Bayesian linear regression model into a Gaussian process regression model with a kernel of size $NKS\times NKS$.
Potentially, these models are useful to identify a good step size $\gamma$ and sample-size $S$ since these parameters are represented in the model.

\section{Comparison and Posterior Predictive Computation}%
\label{sec:post_pred_ngvi}

In \autoref{sec:laplace}, we have introduced three ways to compute the posterior predictive of the Bayesian neural network model.
Based on the derivation of the algorithms and their relation to a Bayesian linear Regression model, we can argue for the right choice of posterior predictive.
LGVA linearizes the neural network \emph{before} taking samples and can therefore be also characterized as a GLM in each inference step.
This can be seen in the updates of LGVA (\autoref{eq:lgva_update_mu} and~\ref{eq:lgva_update_sigma}) where we use a linearized neural network for the log likelihood and therefore have a GLM.
Since OGGN can be derived from LGVA by a crude approximation of the expectation, the same argument holds for OGGN.
Further, OGGN is tightly connected to the Laplace approximation for which we proposed to use the GLM sampling predictive.
Therefore, both LGVA and OGGN should make use of the GLM sampling predictive.

VOGGN works very differently from both LGVA and OGGN since we first sample and then linearize.
That means, during update steps we sample non-linear neural networks and then linearize them individually at the sampled parameter.
This allows the computation of the Hessian approximation for different neural networks in \autoref{eq:vogn_sigma_update}.
Therefore, we obtain $S$ different Jacobians instead of a single one as in LGVA.
In the predictive setting, this would allow to use the NN sampling method.
Notably, VOGGN is the only approximate inference algorithm where we can arguably expect the NN sampling method to work well.
Nonetheless, the NN sampling method is the only one used for prediction with Bayesian neural networks in the past.

\chapter{Experiments}
\label{sec:experiments}

In this chapter, we investigate the behavior of the analyzed and proposed algorithms and validate our hypotheses.
One of the key propositions of this work is that the computation of the predictive distribution should be aligned with the inference algorithm.
Therefore, we investigate experimentally how the combination of approximate inference and the generalized Gauss-Newton method impacts the posterior predictive.
Further, we use the identified generalized linear and Gaussian process models to approximate the marginal likelihood of a neural network, and use the Gaussian process posterior predictive for explainability.
We conduct our experiments on toy regression and classification datasets that allow detailed visualizations.
For the explainability experiment, we use a binary handwritten digit classification task.

In \autoref{fig:datasets}, we illustrate both datasets $\data$ each with $N=150$ data points.
The two-dimensional classification problem is known as ``two moons''~\cite{pedregosa2011scikit}.
Here, we have inputs $\vx_i \in \R^2$ and targets $y_i \in \{0, 1\}$.
The one-dimensional regression task is known as ``Snelson'' named after its inventor~\cite{snelson2007flexible}.
In this case, we have inputs and outputs $x_i, y_i \in \R$.
We further add an artificial gap in this data set to make it more complicated and observe overfitting in line with~\cite{foong2019between, khan2019approximate}.
Both datasets are standard toy problems to benchmark non-linear models like Gaussian processes and neural networks.
For both the regression and classification dataset, we proceed as follows:
first, we select our models using the marginal likelihood.
That means, we find appropriate parameters for the prior and, in the regression example, for the likelihood.
Next, we compare the three posterior predictive distributions for the different inference algorithms (cf. \autoref{sec:posterior_of_gauss_laplace}).
Finally, we make use of the Gaussian process model to understand the predictions of our models.

\begin{figure}[t]
    \centering
    \vspace{-0.7em}
    \begin{subfigure}[b]{0.4\textwidth}
        \centering
        \includegraphics[width=\textwidth]{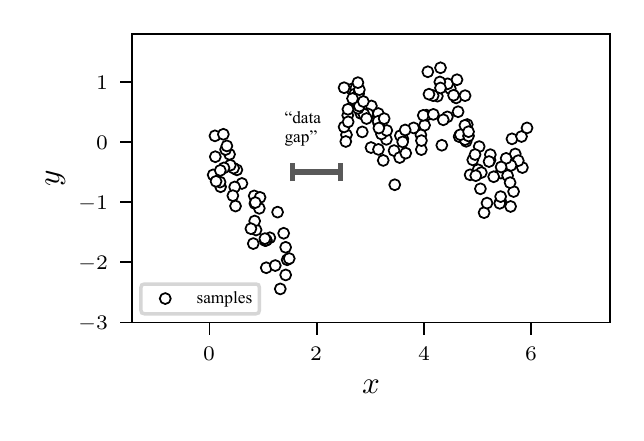}
        \caption{Regression dataset.}
        \label{fig:reg_data}
    \end{subfigure}
    \begin{subfigure}[b]{0.4\textwidth}
        \centering
        \includegraphics[width=\textwidth]{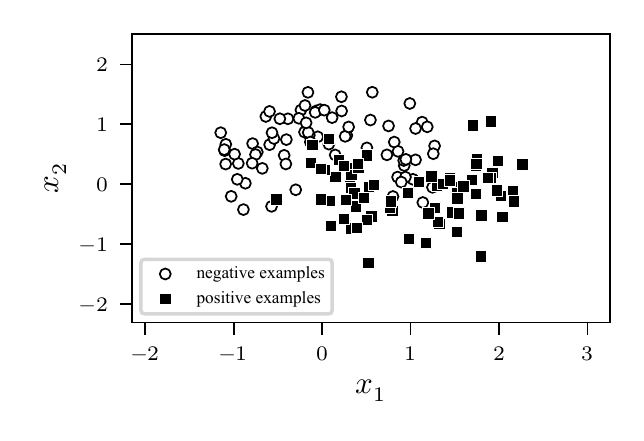}
        \caption{Classification dataset.}
        \label{fig:cls_data}
    \end{subfigure}
    \caption{Visualization of  the two toy example datasets  used in the experimental study.
    Figure~(a) shows the univariate regression dataset Snelson~\cite{snelson2007flexible} with an additional ``data gap''.
    Figure~(b) shows  the two moons classification  dataset that has a noisy decision  boundary and can therefore lead to severe overfitting.}
    \label{fig:datasets}
\end{figure}

For both tasks, we use a standard \emph{multilayer perceptron} with $5$ layers and $25$ hidden units per layer and the \texttt{tanh} activation function.
All layers have bias parameters.
We have parameter vectors $\param \in \R^P$ with $P={2676}$ for the $1$-D regression task and $P={2701}$ for the $2$-D classification task.
The parameter vectors define our neural network function $f(\vx;\param)$ mapping inputs to outputs as depicted in \autoref{fig:GNM_illustration}.
For the binary classification task, we use a Bernoulli likelihood, i.e., we model $Y \sim \textrm{Bernoulli}(f(\vx;\param))$ where the neural network $f$ parameterizes the natural parameter (cf. \autoref{tab:glms}).
In the regression case, we use a Gaussian likelihood with dispersion parameter $\sigma^2$, i.e., we model the response $Y \sim \gauss (f(\vx;\param), \sigma^2)$.
In line with the literature, we use a spherical Gaussian prior with precision $\delta$ on the parameters, i.e., $\param \sim \gauss (\vzero, \delta^{-1} \mI_P)$.
Using the marginal likelihood, we can then optimize the hyperparameter $\delta$ for both problems.
In the regression case, we additionally have the observation noise $\sigma^2$ as hyperparameter.

\section{Model Selection Using Marginal Likelihood}%
\label{sec:modsel_marglik}

We use the Laplace-GGN algorithm introduced in Chapter~\ref{sec:laplace} to compute an approximation to the marginal likelihood.
The marginal likelihood gives evidence to prefer one model over another.
Therefore, it allows us to find suitable hyperparameters $\delta$ and $\sigma^2$.
This procedure is called \emph{empirical Bayes}.
Empirical Bayes is uncommon for neural networks and usually cross-validation schemes are applied~\cite{goodfellow2016deep}, even in Bayesian deep learning~\cite{khan2018fast, zhang2018noisy}.
While empirical Bayes is uncommon for neural networks, it has been explored before in the context of Laplace and Gaussian variational approximations~\cite{foresee1997gauss, wu2019deterministic}.
In particular, the method of \citet{foresee1997gauss} is the same as the one presented here in the case of a Gaussian likelihood.

For both datasets, we have $N=150$ training samples and a test set $\data_\textrm{test}$ with $1000$ additional input output pairs.
This allows us to estimate the generalization error.
For the generalization error, we use the average log likelihood on the test data at the MAP.
Let $p$ be the likelihood of the corresponding model.
Then, we have for the average log likelihood at the MAP
\begin{equation}
  \label{eq:log_loss}
  \ell \ell = \frac1{|\data_\textrm{test}|} \sum_{(\vx_i, y_i) \in \data_\textrm{test}} \log p(y_i \given f(\vx_i;\param_\textrm{MAP})),
\end{equation}
where $\vx_i$ is a scalar in the regression dataset.
Ideally, the marginal likelihood approximation suggests the same optimal parameters as the test log likelihood.
In the regression problem, we choose the hyperparameter ranges $\sigma^2 \in [0.001, 10]$ and $\delta \in [0.0001, 100]$.
For classification, we choose $\delta \in [0.01, 100]$.
Using this range of hyperparameters, we train the one neural networks for each parameter setting until convergence to a MAP estimate.
For training the MAP objective, we use the Adam optimizer~\cite{kingma2014adam}.
At the MAP, we compute the Laplace-GGN approximation to the marginal likelihood, i.e. , to the local generalized linear model.
For the predictive distribution, we therefore also choose the GLM sampling variant.

\begin{figure}[t]
  \vspace{-3.2em}
    \centering
    \includegraphics[width=\textwidth]{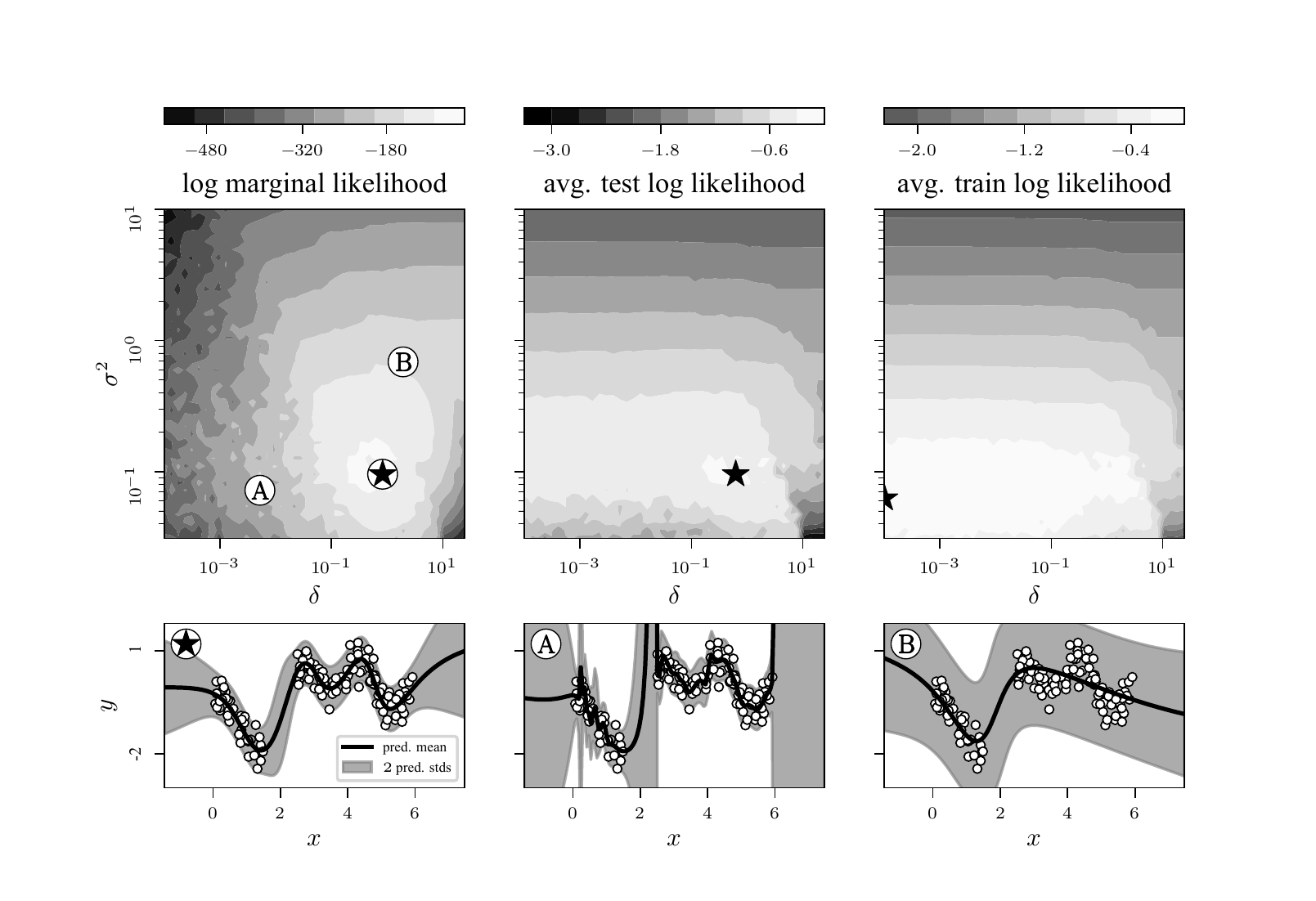}
    \caption{Marginal likelihood with respect to observation noise $\sigma^2$ and prior precision $\delta$ of a neural network model on a toy regression problem with example posterior predictive distributions.
    In the top, the marginal likelihood approximation of a neural network model due to the Laplace-GGN in comparison to average test and train log likelihood is displayed.
    The marginal likelihood provides a robust way to choose hyperparameters and is in line with the test log likelihood.
    In contrast, the optimal likelihood on the training data can go to zero which leads to an overfitting model.
    In the bottom, the optimal posterior predictive due to the marginal likelihood ($\star$) and examples of overfitting (A) and underfitting (B) are visualized.
    We show the posterior predictive mean and two standard deviations.
    }
    \label{fig:reg_marglik}
\end{figure}

\autoref{fig:reg_marglik} shows the marginal likelihoods for different hyperparameters on the regression problem.
The optimal hyperparameters found using the marginal likelihood are very close those that generalize the best according to the test log likelihood.
In contrast, we can see that the neural network can become too expressive and overfit when we have weak regularization.
Weak regularization corresponds to a small prior precision $\delta$ and can lead to a complex predictive function that overfits to individual training data points.
\autoref{fig:reg_marglik} further depicts three generalized linear model predictives: the optimal, an overfitting, and an underfitting model.
We see that the optimal model also visually trades off between complexity and simplicity while the overfitting model is overly complex and fits the noise.
In contrast, the underfitting model fails to match the shape of the underlying data generating function.
We identify the optimal hyperparameters $\delta=0.63$ and $\sigma^2=0.1$.
Notably, the observations are generated with a similar noise variance of $0.09$.

\begin{figure}[t]
    \centering
    \includegraphics[width=\textwidth]{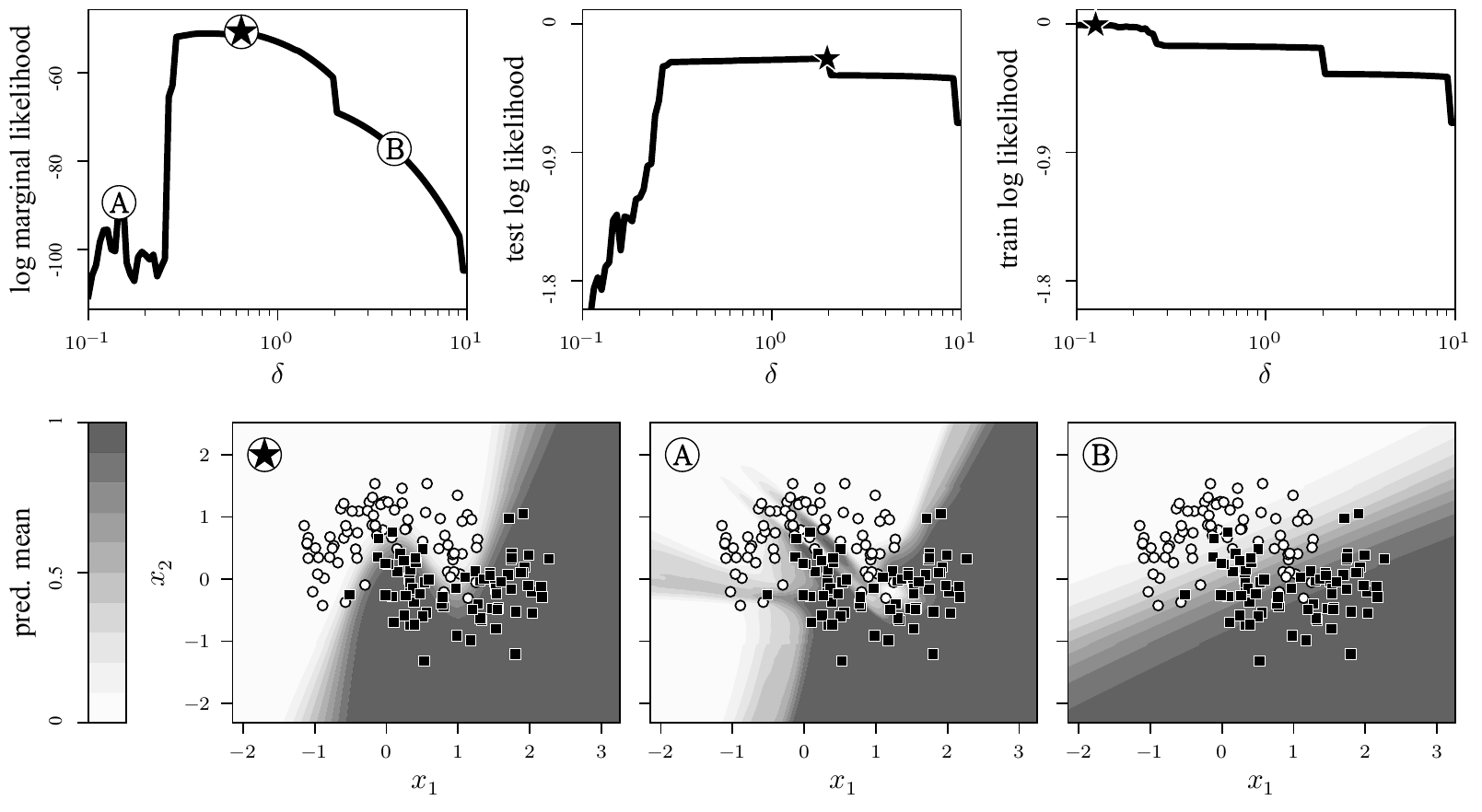}
    \caption{Marginal likelihood with respect to prior precision $\delta$ of a neural network model for classification with examples of the corresponding predictive mean.
    In the top, we compare the marginal likelihood to the log likelihood on the testing and training dataset.
    The marginal likelihood identifies the optimal range of prior precision between $0.3$ and $2$ in line with the test log likelihood.
    In the figure displaying the log likelihood on the training data, the model overfits clearly for small prior precision values.
  In the bottom, the mean of the posterior predictive due to sampling from the generalized linear model is visualized for the model with optimal marginal likelihood ($\star$) and an overfitting (A) and underfitting (B) model.
  The overfitting model achieves almost zero misclassifications in the training data due to an overly complex decision boundary.
    }
    \label{fig:class_marglik}
\end{figure}

In \autoref{fig:class_marglik}, we conduct the same analysis for the classification problem and obtain similar results:
According to the marginal likelihood, the range of optimal hyperparameters lies between $0.02$ and $2$ which matches the plateau of the test log likelihood accurately.
The marginal likelihood identifies the optimal hyperparameter $\delta=0.13$.
The training log likelihood goes to zero for weak regularization indicating a perfect fit and correct prediction of each training data point.
However, both the marginal likelihood and test likelihood reject such an overfitting model.
\autoref{fig:class_marglik} further displays such a model in comparison to the optimal model according to the marginal likelihood and an underfitting model that only represents a linear decision boundary.
As in the regression problem, the optimal model also exhibits the best uncertainty around the decision boundary as the boundary becomes wider away from the data.
In the next section, we focus especially on the properties of the posterior predictive distribution.
We use the optimal hyperparameters identified in this section.

\section{Posterior Predictive Distributions}%
\label{sec:post_pred_exp}

We have introduced three ways to obtain an approximate posterior predictive distribution for a Bayesian neural network:
NN sampling, GLM sampling, and BLR inference (cf. \autoref{sec:predictive_laplace}).
Disentangling the GGN and approximate inference, we posed the hypothesis that only VOGGN can lead to a stable predictive using NN sampling.
We investigate this hypothesis here.
We use the optimal hyperparameters found in the previous section.

We train the neural network models using the three natural-gradient variational inference algorithms introduced in Chapter~\ref{sec:vi}.
We use the step size $\beta=0.999$, a single Monte Carlo sample from the posterior approximation for VOGGN and LGVA per step, and initialize the posterior covariance to $\mSigma=0.1 \mI_P$.
We use the same randomly initialized mean $\vmu$ for all algorithms and train until convergence.
For the Laplace-GGN (L-GGN) algorithm, we again use Adam to obtain a MAP estimate and then apply the Laplace-GGN posterior approximation.
For the NN and GLM sampling posterior predictive approximations, we use $1000$ Monte Carlo samples.
In the regression case, we have a Gaussian likelihood and therefore the GLM coincides with the BLR model.
In the classification case, the GLM corresponds to a Bayesian logistic regression model.

In \autoref{fig:reg_predict}, we display the posterior predictive distributions on the regression task for all algorithms using the NN sampling and GLM/BLR inference method.
Additionally, we display posterior predictive samples to understand not only the marginal mean and variance but also the joint predictive distribution.
The BLR posterior predictive works consistently across all approximate inference methods and provides reasonable uncertainty estimates.
Between the methods, there is no significant difference using the BLR predictive.
In contrast, the NN sampling posterior predictive fails for the Laplace-GGN and the Laplace-like OGGN since the predictive mean is inaccurate and the variance exceedingly high.
The posterior approximation due to LGVA also leads to overestimated predictive uncertainties and an inaccurate mean.
Only VOGGN exhibits a posterior predictive that is reasonable and similar to the GLM variant of the VOGGN posterior approximation.
On this toy example, we can clearly see that we should predict based on the underlying model that we infer.

\begin{figure}[H]
    \centering
    \vspace{-1em}
    \begin{subfigure}[b]{\textwidth}
        \centering
        \includegraphics[width=\textwidth]{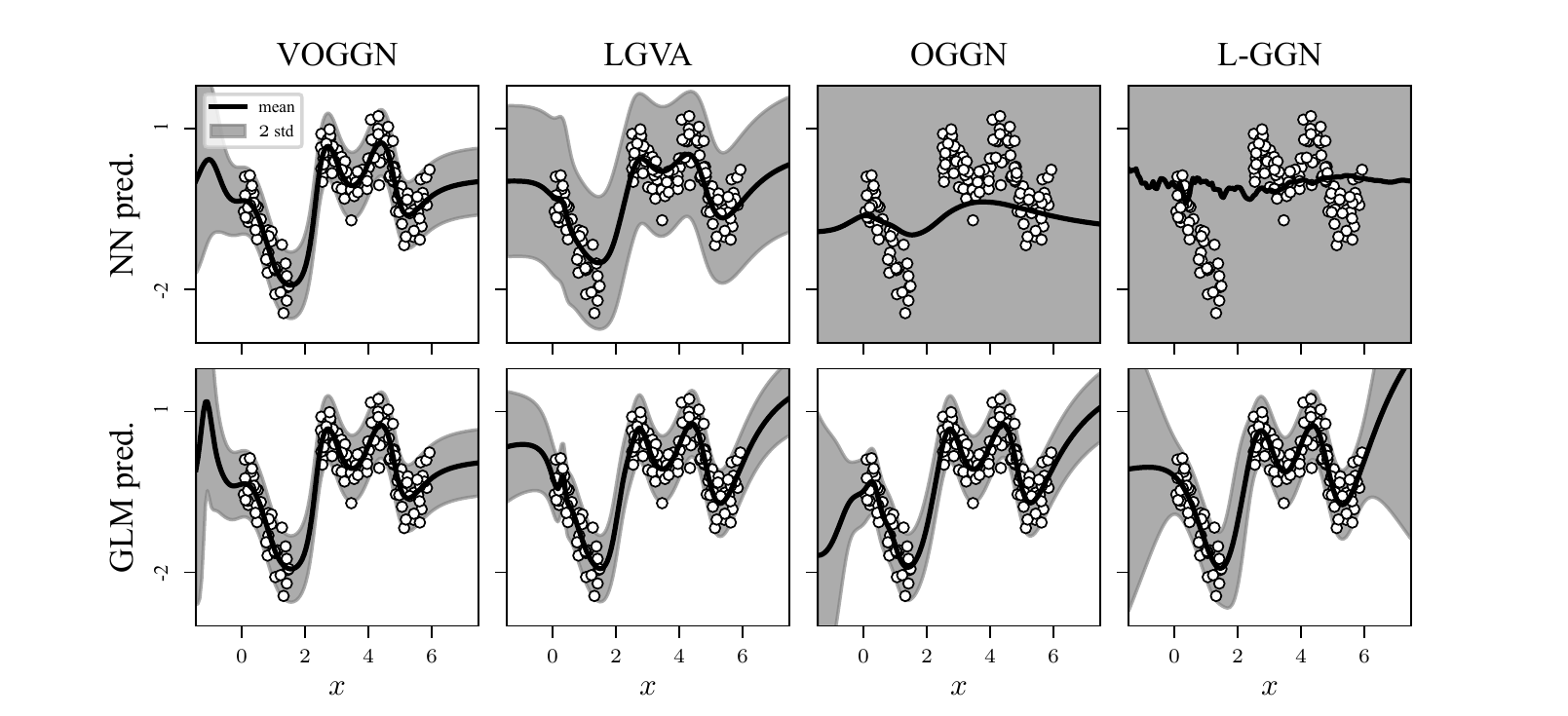}
        \caption{Posterior predictive distribution (NN vs. GLM).}
        \label{fig:reg_predist}
    \end{subfigure}
    \vfill
    \begin{subfigure}[b]{\textwidth}
        \centering
        \includegraphics[width=\textwidth]{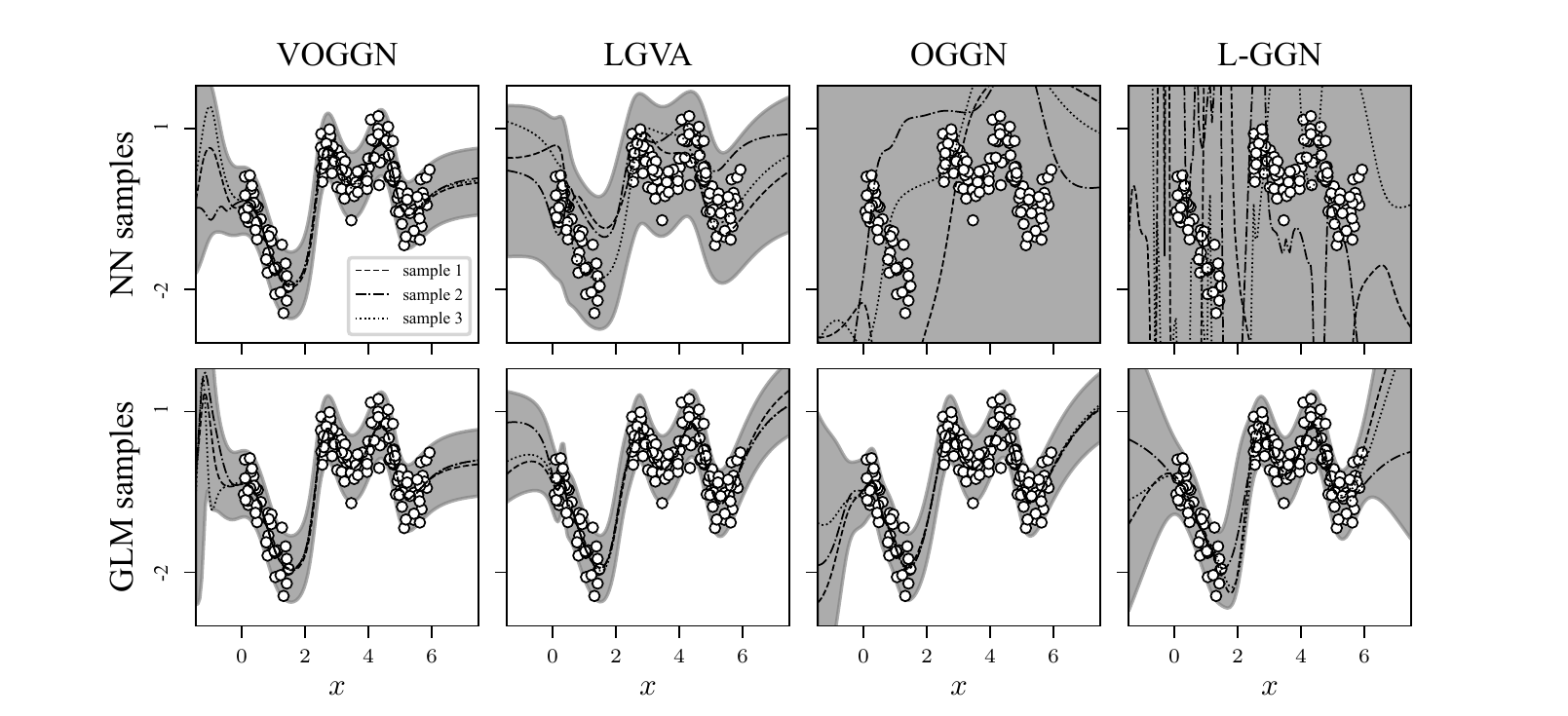}
        \caption{Posterior predictive samples (NN vs. GLM).}
        \label{fig:reg_pred_samples}
    \end{subfigure}
    \caption{Comparison of posterior predictive by NN sampling and GLM sampling for four approximate inference algorithms.
    The top row shows prediction due to NN sampling and the bottom row shows GLM sampling.
    In~(a), we show the predictive mean and standard deviation.
    Figure~(b) displays three posterior predictive samples.
    For the Gaussian likelihood used here, GLM sampling is equivalent to the exact Bayesian linear regression predictive.
    Using the GLM, the predictions are reliable for all algorithms.
    In contrast, only VOGGN gives reasonable results when predicting by sampling neural networks.
    Since LGVA samples like VOGGN, it can still give reasonable predictions using NN sampling.
     OGGN and L-GGN do not work in this case.
    }
    \label{fig:reg_predict}
\end{figure}

\begin{figure}[H]
    \centering
    \begin{subfigure}[b]{\textwidth}
        \centering
        \includegraphics[width=\textwidth]{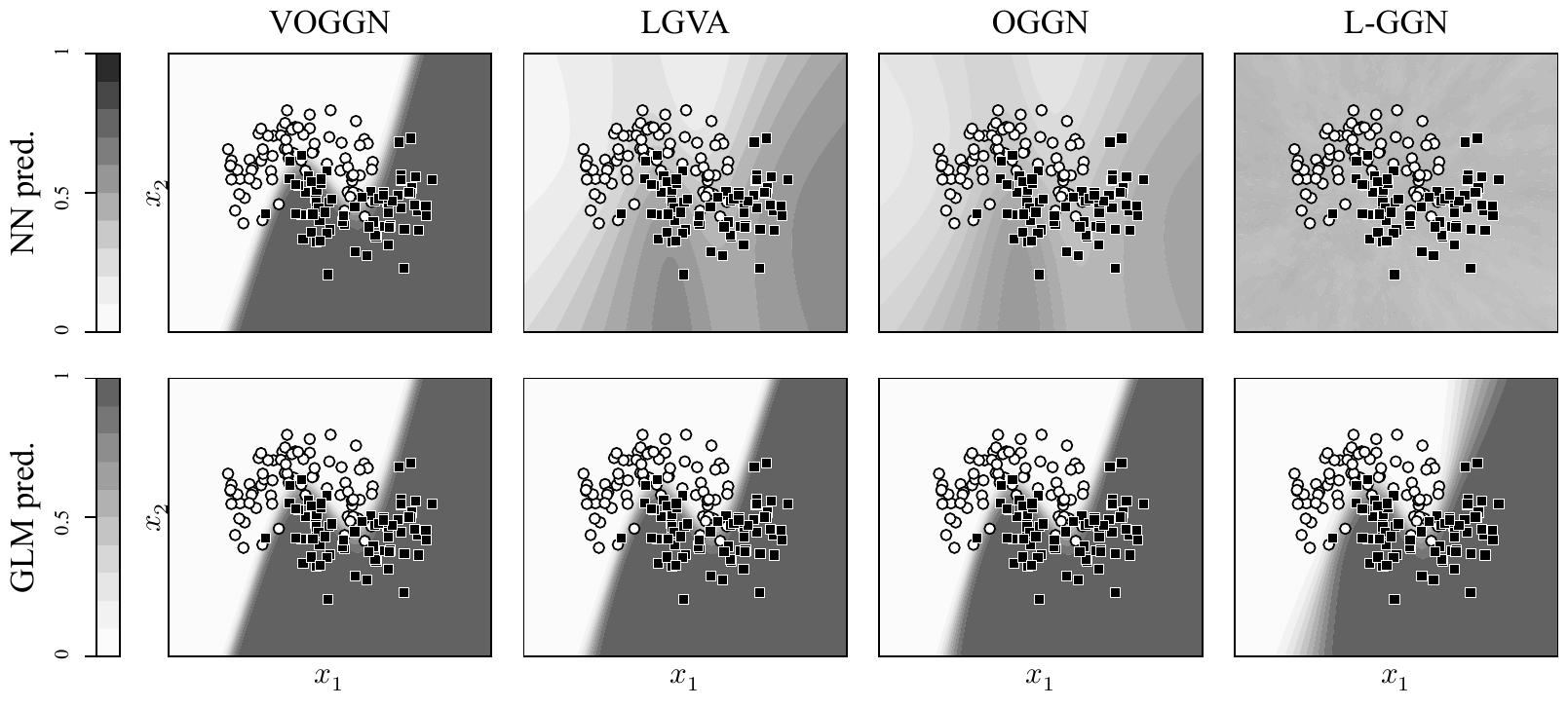}
        \caption{Posterior predictive mean (NN vs. GLM).}
        \label{fig:cls_predist}
    \end{subfigure}
    \vfill
    \begin{subfigure}[b]{\textwidth}
        \centering
        \includegraphics[width=\textwidth]{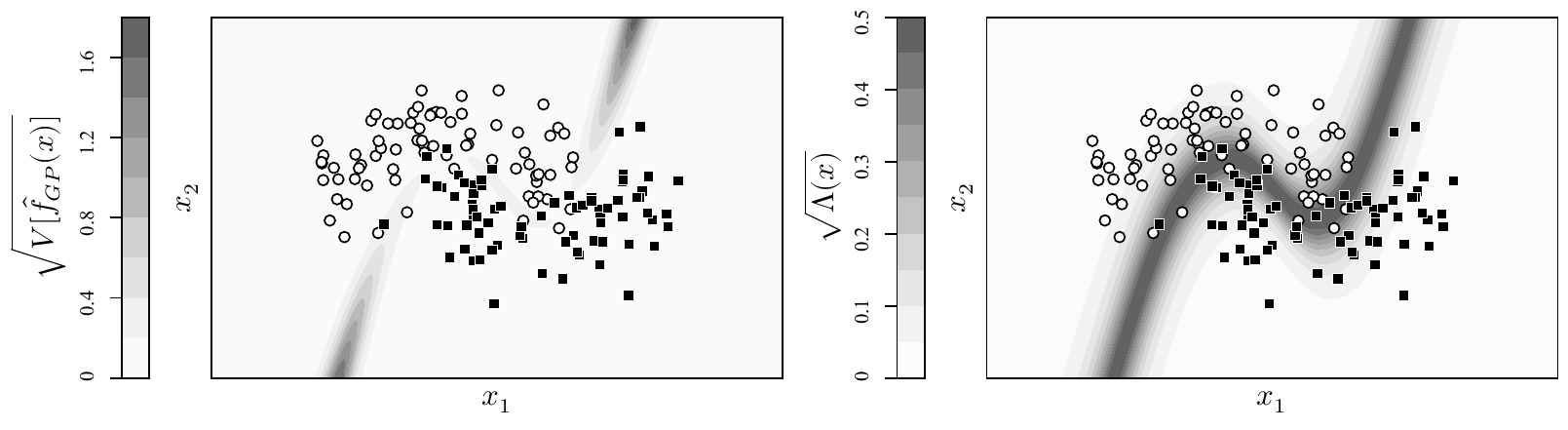}
        \caption{Posterior predictive uncertainty due to Bayesian linear regression.}
        \label{fig:cls_blruct}
    \end{subfigure}
    \caption{Comparison of NN sampling, GLM sampling, and Bayesian linear regression posterior predictive.
    Figure~(a) shows the posterior predictive mean by NN sampling in the top and GLM sampling in the bottom.
    Only VOGGN maintains good performance using NN sampling.
    All four algorithms show optimal performance when we use GLM sampling to predict.
    In Figure~(b), we look into the model uncertainty and observation noise due to the Bayesian linear or GP regression model.
    While the observation noise has simply the variance of the Bernoulli response variables, the model uncertainty on the left grows further away from the data and is relatively low where the decision boundary is supported by data.}
    \label{fig:cls_predict}
\end{figure}

In the classification example, we make similar observations:
\autoref{fig:cls_predist} depicts the posterior predictives of all methods using NN and GLM sampling.
Again, only VOGGN leads to a posterior works for both NN and GLM sampling.
In line with the regression problem, the posterior predictive due to the GLM performs consistently across all methods and provides similar predictive distributions for posterior approximations.
Prediction using NN sampling is again only viable using the VOGGN posterior approximation.
LGVA and OGGN show signs of a decision boundary but the quality is worse than that of the GLM predictive.
For the Laplace-GGN posterior approximation, NN sampling leads to uniform predictions and therefore no decision boundary.
In stark contrast, the corresponding GLM sampling predictive yields an optimal predictive model.

In \autoref{fig:cls_blruct}, we analyze the Bayesian linear regression model that is inferred exactly when we use the Laplace or Gaussian variational approximation.
In particular, we look into the uncertainty in the posterior Gaussian process and the observation noises.
The observation noise corresponds to the variance of the modelled Bernoulli random variable and therefore highlights the decision boundary.
However, the uncertainty in the Gaussian process is low around the decision boundary and instead grows away from the data.
This property could be useful to extend the decision boundary as it is desired for example in \emph{active learning} or \emph{Bayesian optimization}.
Interestingly, the model has high model certainty outside of the data as long as it is far away from the inferred decision boundary.
This is expected since the kernel of neural networks is typically not stationary~\cite{lee2017deep, neal1993probabilistic, williams1998computation}.

The experiments on posterior predictive distributions show clearly that it is important to understand the approximations used in Bayesian deep learning.
Having understood the impact of the GGN on individual approximate inference algorithms, we can choose the right posterior predictive and substantially improve the performance.
In fact, the pathology of prediction with NN sampling using the Laplace-GGN posterior approximation has been already pointed out in the literature~\cite{foong2019between, ritter2018scalable}.
\citet{ritter2018scalable} argued that conditioning and the ratio of data points and parameters leads to this problem.
The GLM sampling method fixes this problem and it turns out that it works reliably for few data points ($N=150$) and comparatively many parameters ($P\geq 2000$).
The problem is therefore not due to conditioning but a predictive procedure that does not align with the inference and approximation methods used.

\section{Function-Space Neural Network Inference for Explainability}%
\label{sec:kernel_explain}

In this section, we use the Gaussian process formulation of the Laplace-GGN approximation to explain neural network predictions.
In \autoref{sec:laplace}, we have shown that the generalized Gauss-Newton gives rise to a generalized linear or generalized Gaussian process model.
Gaussian process models are instance-based learning algorithms, i.e., they make predictions directly based on the training data.
Therefore, we can understand predictions by looking into the training data points responsible for them.
In particular, we try to understand predictions of a convolutional neural network on the binary classification task of distinguishing handwritten digits $4$ and $9$.

The predictive mean of a Gaussian process regression model can be written as an inner product of the kernel between a test point and the training data and an importance vector~\cite{rasmussen2003gaussian}.
We have a vector $\va \in \R^N$ that gives an importance factor to each training data point that depends on the likelihood of the generalized GP model.
Further, we have the kernel vector $\vk \in \R^N$ with entries computed by the kernel $\kappa(\vx_*, \vx_i)$ between a test data point $\vx_*$ and the training dataset.
Then, the predictive mean of the GP posterior mean on a new data point can be written as
\begin{equation}
  \label{eq:gp_post_pred}
  f_* = \sum_{i=1}^N a_i k(\vx_*, \vx_i).
\end{equation}
This allows us to understand the prediction by looking into individual entries of $\va$ and $k(\vx_*, \vx_i)$.
In particular, for a generalized Gaussian process, the \emph{importances} $\va$ can be related to the residuals $\vr(\vy, \vf)$ of the log likelihood, i.e., the first derivative.
In the classification case, we have $a_i = \nabla_\vf \log p(\vy \given \vf(\vx_i))$ where $\vf$ is, for example, our Gaussian process formulation of the neural network and $p$ denotes a Bernoulli likelihood.
The kernel $k(\vx_*, \vx_i)$ quantifies the \emph{similarity} between a test and training data point according to the feature map or kernel function.
Therefore, it allows to identify similar training data points that lead to a particular prediction.

We apply the Laplace-GGN to a convolutional neural network.
We train the neural network on the digits $4$ and $9$ which constitutes the hardest binary classification task on the MNIST dataset~\cite{lecun1998gradient}.
In particular, we randomly select $3000$ samples for training.
The network has 2 convolutional layers each followed by a \texttt{ReLU} activation function and \texttt{MaxPooling}.
The last three layers are linear and also use the \texttt{ReLU} activation function.
In total, the network has $P=4587$ parameters.
We use hyperparameter $\delta=10$. 
Since the Laplace-GGN approximation can be equivalently cast as the Laplace approximation in a generalized Gaussian process model, we obtain for $\va$ the residuals $r(y_i, f(\vx_i; \param_{\textrm{MAP}}))$.
The kernel between a new test and a training data point is given by $\kappa(\vx_*, \vx) = \delta^{-1} \mJ(\vx_*; \param_*) \mJ(\vx;\param_*)$.

In \autoref{fig:kernel_explain}, we analyze both quantities for the binary MNIST problem.
Since the model achieves perfect classification on the training data, the residuals are not bigger than $0.2$ in \autoref{fig:mnist_kernel_alphas}.
The data points with the highest absolute residuals can be understood as boundary points and depict particularly notable examples:
$4$s that look like $9$s and vice versa.
These data points have a high impact in decisions since their residuals play an important role in the predictive formulation in \autoref{eq:gp_post_pred}.
In the middle range, we have ordinary examples that are easy to distinguish.
In \autoref{fig:mnist_kernel}, we display one correct and two incorrect predictions along with examples of training data points and the kernel value between these samples and the test input.
For the incorrect predictions, we find that an example of the opposite class strongly aligns with the test input (see rank 1).
In both cases, these samples are also boundary points that highly influence the decision.
This ultimately leads to a misclassification.

Understanding predictions of neural networks using an instance-based approach can potentially help make decisions more robust, improve explainability, and identify problematic training data.
It would further be interesting to understand how different neural network architectures induce different feature maps and therefore Jacobians.
Potentially, some choices lead to good inductive biases for Bayesian neural networks.

\begin{figure}
    \centering
    \begin{subfigure}[b]{\textwidth}
        \centering
        \hspace{-1em}%
        \includegraphics[width=0.9\textwidth]{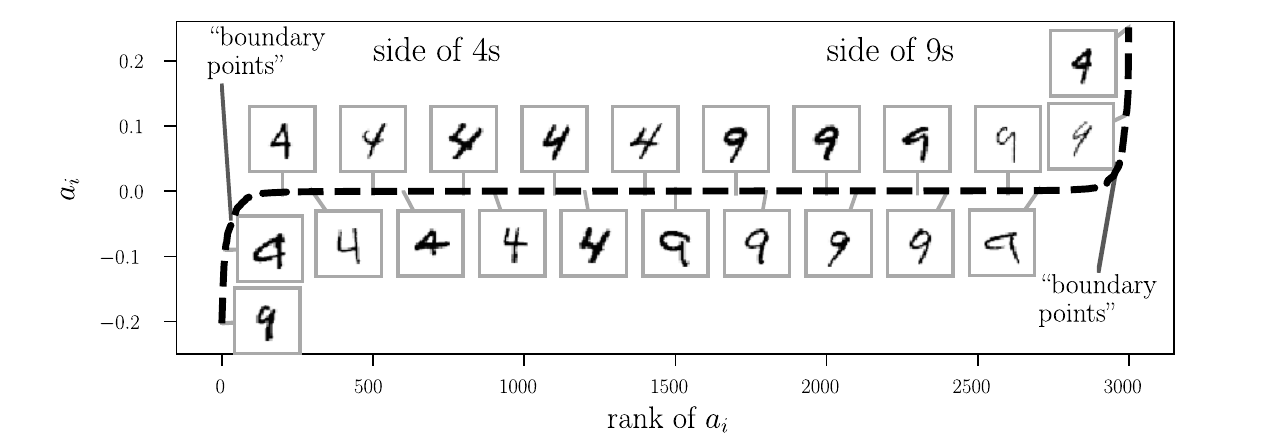}
        \caption{Sorted training data importances $\va$ with examples.}
        \label{fig:mnist_kernel_alphas}
    \end{subfigure}
    \vfill
    \begin{subfigure}[b]{\textwidth}
        \centering
        \vspace{1em}
        \includegraphics[width=\textwidth]{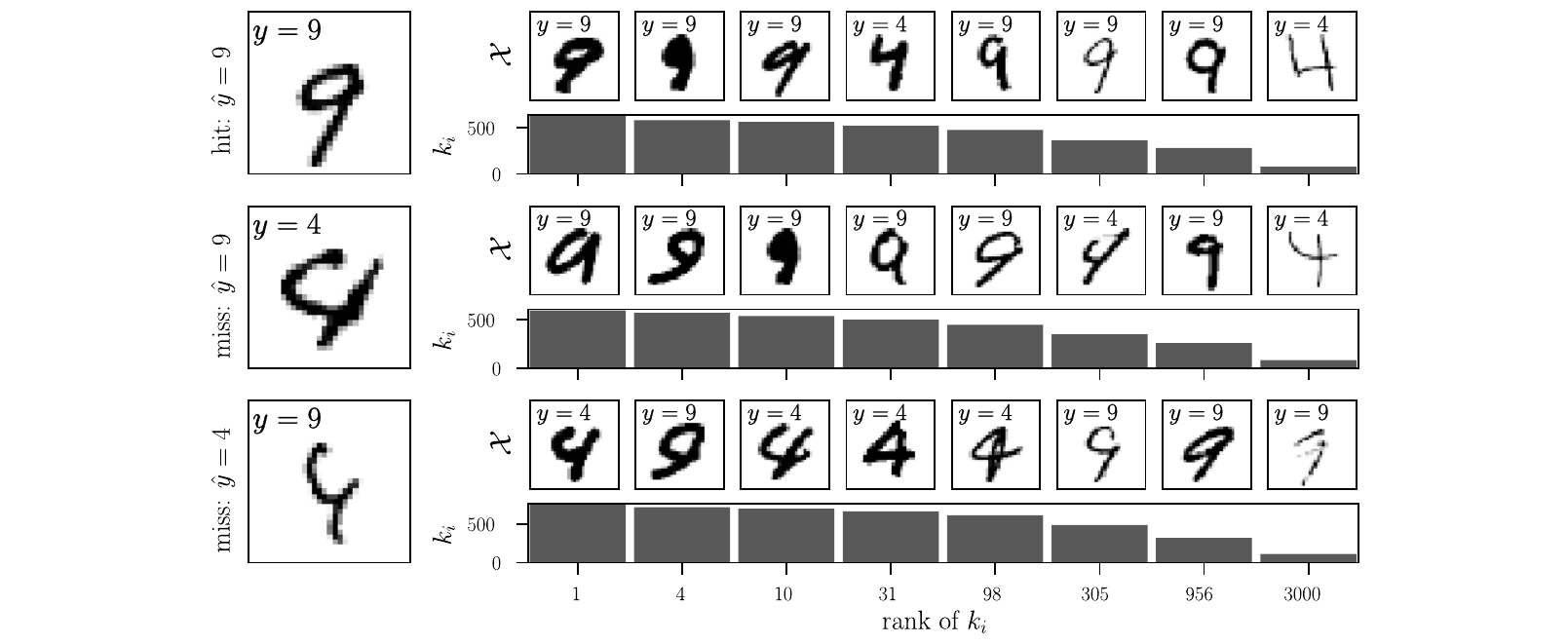}
      \vspace{-0.2em}
        \caption{Classification on test inputs and similarity to training data.}
        \label{fig:mnist_kernel}
    \end{subfigure}
    \caption{Understanding neural network predictions using a Gaussian process view: in Figure~(a), the elements of the importance vector $\va$ are displayed in order and with corresponding examples.
    In the left and right border, we can identify boundary points, i.e., data points that are nearly misclassified:
    in fact, on the left we have $4$s that look like $9$s and vice versa on the right.
    In Figure~(b), the similarity between training and test data points is used to understand predictions.
    On the left, test data points and their predictions $\hat{y}$ are displayed.
    On the right, we list $8$ training data points and their labels together with the corresponding similarity to the test image.
    The kernel vector $\vk_*$ for a test image $\vx_*$ determines the decision together with the vector $\va$.
    We can see that misclassified examples correlate highly with examples of the wrong class due to the learned feature map of the convolutional neural network.
    Further, the misclassified inputs both show high similarity to points at the decision boundary that highly impact the final prediction.
   }
    \label{fig:kernel_explain}
\end{figure}

\chapter{Discussion and Future Directions}
\label{sec:discussion}
In this chapter, we discuss the work related to this thesis and conclude the results with a future outlook.
Throughout this work, we have referred to related literature where appropriate.
Here, all references are summarized in a single place.
The conclusion contains a short review of the presented results and discusses possible directions for future work.

\section{Related Work}%
\label{sec:related_work}

This thesis complements and extends the paper ``Approximate Inference Turns Deep Networks into Gaussian Processes''~\cite{khan2019approximate}.
The focus and standpoint of the present work is different.
We try to disentangle the generalized Gauss-Newton method and approximate inference in Bayesian deep learning to gain theoretical understanding and practical advantages.
In contrast, the prior work focuses on obtaining a Gaussian process representation of neural networks directly from the Bayesian deep learning algorithm~\cite{khan2019approximate}.
In contrast, the present work can therefore identify an intermediary generalized linear and Gaussian process model that proves to be useful to derive new posterior predictive, marginal likelihood, and inference algorithms for Bayesian neural networks.
The Bayesian linear and Gaussian process regression models of both works are equivalent up to reparameterization.
The reparameterization plays a major role to practically apply the identified connection.
\citet{khan2019approximate} identify a linear and GP regression model in a transformed data space, which, in its original form, cannot replace the neural network model.
Next to the two algorithms VOGGN and OGGN introduced by the prior work, the present work additionally introduces the LGVA algorithm that lies between Laplace and variational approximation.
The experiments presented here are different since they serve the purpose of particularly showing the impact of the GGN on approximate inference and are based mostly on the novel generalized linear and Gaussian process model formulations.
The prior work focuses on the Gaussian process regression formulation that is obtained.
Therefore, their marginal  likelihood formulation only works for a Gaussian likelihood.

Recently, there has been a surge of interest in the connection of neural network training or inference and Gaussian processes.
\citet{williams1998computation} and \citet{neal2012bayesian} already connected Bayesian neural networks to Gaussian processes in the $90$s.
In particular for a single hidden layer of infinite width, one could show that, under a Gaussian prior, the neural network function can be characterized as a non-stationary Gaussian process~\cite{williams1998computation}.
This result has recently been extended to other architectures, activations functions, and depths~\cite{lee2017deep}.
The derivation of the neural tangent kernel that characterizes the training of a neural network in function space~\cite{jacot2018neural} has again sparked the interest in relating neural networks and Gaussian processes.
Based on the work of \citet{jacot2018neural}, \citet{lee2019wide} showed that infinite width neural networks can be understood as linearized neural networks and therefore be connected to Gaussian process inference if we pose a prior on the parameters.
They further found empirical evidence that this connection even holds in the finite setting.
In this work, we do not analyze the probabilistic neural network model theoretically but rather the combination of this model with a corresponding practical algorithm.
Notably, this gives us similar results and allows to relate neural network inference with Gaussian process inference.
In contrast, we do not need to take the limits but obtain similar results due to the combination of the GGN and a Gaussian posterior approximation.

The combination of the generalized Gauss-Newton and approximate inference for Bayesian deep learning is very common.
The GGN is mostly applied to approximate the Hessian of the log likelihood.
In some cases, it is however used to approximate the Fisher information matrix required for natural gradient descent.
For a discussion this, we refer the reader to the recent work of \citet{kunstner2019limitations}.
The combination of Laplace approximation and Gauss-Newton is popular and has already been suggested for least-squares regression with neural networks in the $90$s~\cite{foresee1997gauss}.
Modern large-scale Bayesian deep learning algorithms based on the Laplace approximation do not use the full GGN approximation but rather diagonal or factorized variants~\cite{ritter2018scalable} and have successfully been applied to continual learning with neural networks~\cite{ritter2018online}.
The Gaussian variational posterior approximation is more popular than the Laplace approximations as it promises to be more precise~\cite{bishop2006pattern}.
In particular, the combination with natural gradients~\cite{amari1998natural}, efficient and numerically stable inference algorithms are possible~\cite{khan2018fast, osawa2019practical, zhang2018noisy}.
Prior to their work, the GVA for neural networks mostly relied on backpropagation and the reparameterization trick.
The posterior approximation of these algorithms is often unstable and depends heavily on the hyperparameters~\cite{blundell2015weight, foong2019between}.
Notably, the state-of-the-art results obtained in Bayesian deep learning all rely on a combination of Gaussian posterior approximation and generalized Gauss-Newton method~\cite{khan2018fast, mishkin2018slang, osawa2019practical, ritter2018scalable, zhang2018noisy}.
This work is the first to analyze the interplay of the generalized Gauss-Newton and approximate inference in detail.
All above works start from the motivation of a Bayesian deep learning algorithm and use approximations along the derivation of a new algorithm.
In contrast, we discuss the impact of individual approximation choices and the impact on the underlying probabilistic model.
Further, prior work exclusively relies on NN sampling to approximate the posterior predictive which we find to be unreliable in many cases.
The detailed present discussion allows to  fix this problem due to two new  posterior  predictive  distributions that are in line with the inferred model.

Problems with the posterior predictive due to a Laplace-GGN approximation to a neural network are known~\cite{ritter2018scalable, foong2019between}.
Recently, \citet{foong2019between} discovered that the BLR predictive works consistently better than NN sampling for univariate Gaussian likelihoods.
This work supports this observation theoretically and further extends the result to other likelihoods  using the stable GLM sampling prediction method.
Further, we justify the same predictive procedure for Gaussian variational approximations.
The  form of the marginal likelihood based on the Bayesian linear regression model is equivalent to the evidence approximation due to \emph{Occam's razor}~\cite{mackay1995probable}.
The difference is that we propose to compute the model evidence of the underlying GLM and obtain an additional Gaussian process variant.

The seminal work by \citet{wedderburn1974quasi} studies the impact of the generalized Gauss-Newton method on maximum likelihood estimation for generalized linear models.
In particular, he finds that an optimization step using the Gauss-Newton method of a non-linear model with GLM likelihood can be cast as a least-squares regression problem.
This least-squares regression problem specifies an adjusted noise variance similar to the one we find here.
He was the first to generalize the Gauss-Newton algorithm to GLM likelihoods and therefore  defined the generalized Gauss-Newton algorithm.
He shows that the generalized Gauss-Newton method applied to maximum-likelihood estimation requires iterative solutions of a least-squares regression problem.
In our case we have a similar Bayesian linear regression model that further allows computation of other quantities and can be related to Gaussian processes.
While we have specifically  focused on Bayesian neural networks, the results presented here also hold for general parametric functions $\vf(\vx;\param)$ that are at least once differentiable in the parameter $\param$.
This is precisely the case \citet{wedderburn1974quasi} worked with.

\section{Conclusion}%
\label{sec:conclusion}

In this thesis, we have disentangled the generalized Gauss-Newton and approximate inference methods in Bayesian deep learning algorithms.
The individual analysis of both methods has shed new light on these algorithms:
the generalized Gauss-Newton algorithm turns the neural network model into a local generalized linear model.
Further, approximating the posterior of the GLM requires the exact solution to a Bayesian linear regression problem which gives us the posterior approximation to the neural network model.
We have made use of this new understanding to improve the posterior predictive of Bayesian neural networks, enable empirical Bayes for tuning hyperparameters using the generalized linear model formulation, and identify a function-space posterior approximation for neural networks.
The theoretical findings are supported by experiments on simple toy datasets that enable detailed investigation.
The present work enables future research in quantifying the accuracy of approximate Bayesian inference using the two identified stages and in applying  the derived quantities based on the underlying generalized linear and Gaussian process models to real and larger datasets.

For the Laplace approximation, we have shown that the common combination with the generalized Gauss-Newton optimization method can be understood in two stages:
the first stage consists of obtaining a MAP estimate of the neural network and using the generalized Gauss-Newton that implicitly linearizes the neural network at the MAP.
This turns the neural network model into a generalized linear model.
The second step is to apply the Laplace approximation.
We have found that the Laplace approximation implicitly moment-matches the generalized linear model likelihood to a Gaussian likelihood.
This turns the GLM into a Bayesian linear regression model which can be solved exactly and gives us the neural network posterior approximation.
Both intermediary models, the generalized linear and Bayesian linear regression model, have Gaussian process counterparts that have the same marginal likelihood and posterior predictive.
This connection allows a function-space Laplace-GGN posterior approximation which is useful when the number of parameter greatly exceeds the number of data points.
Approximate inference is therefore conducted in the underlying  generalized linear or Gaussian process model.
Hence, we have argued that these models should be used for the posterior predictive and marginal likelihood computation.
The regression models that we need to solve exactly also provide robust predictions but have the wrong likelihood.
In particular, we presented the GLM sampling and closed-form BLR predictive distributions that complement the common NN sampling method.
The experiments provide practical evidence that the GLM sampling method provides more consistent posterior distributions than NN sampling.
Further, selecting hyperparameters using the marginal likelihood approximation to the GLM provides a way to optimize neural network hyperparameters only based on the training data.
Lastly, we have shown that the Gaussian process view can be useful for understanding the predictions of neural networks.

The generalized Gauss-Newton method is used in many recent variational inference algorithms for Bayesian deep learning and it is therefore critical to understand its impact.
Variational inference differs from the Laplace approximation in one critical point:
we have to compute an expectation to obtain the updates or characterize stationarity (cf. \autoref{sec:ABI}).
In practice, we approximate this expectation by sampling.
Therefore, we identified two ways to make use of the GGN:
we either sample first and then apply the GGN or vice versa.
Applying the GGN first led to the newly introduced LGVA algorithm.
Sampling first led to the VOGGN algorithm that we analyzed as a prototype for most variational algorithms for neural networks.
The third algorithm, OGGN, is a crude approximation to the LGVA that avoids  sampling.
We obtained the following understanding of the LGVA algorithm:
using the GGN first leads to a generalized linear model as for the Laplace approximation.
A variational inference step can then be formulated as a Bayesian linear or Gaussian process regression model.
This model made clear that the LGVA algorithm does not sample neural networks but only generalized linear models.
We have therefore argued for the GLM sampling method for the posterior predictive.
For the VOGGN algorithm we found that that underlying regression problem we solve in each step is specified by different neural network feature maps.
This corresponds to sampling neural networks and led to the conclusion that the posterior approximation inferred with VOGGN can make use of the NN sampling posterior predictive.
Experimentally, we found results that support this hypothesis.
In fact, VOGGN provides the only posterior approximation that allowed to use NN sampling and obtain reasonable predictions consistently.
Interestingly, we could not determine the single best algorithm in the experiments.
Both the Laplace and variational algorithms provide equally good  predictions and uncertainty estimates using the GLM sampling method introduced here.

\section{Future Work}%
\label{sec:future_work}

The results presented in this thesis can be used for future theoretical or applied research.
On the theoretical side, the isolation of the generalized Gauss-Newton method and approximate inference in Bayesian deep learning could be useful to investigate convergence, priors, and feature maps.
The experiments presented here only serve  the purpose of understanding the theoretical results and validate hypotheses.
An interesting future direction would therefore be to investigate these results on real or larger-scale data and for other applications.

The individual understanding of the generalized Gauss-Newton and approximate inference for Bayesian deep learning could be useful to jointly specify how accurate these methods really are.
In particular, the disentangled understanding also sheds new light on the prior in Bayesian deep learning.
In the Laplace approximation, it serves as a regularizer during MAP estimation and as a distribution in approximate inference.
However, since we infer a GLM due to the GGN, the prior potentially plays another role during inference.
Since the MAP plays an important role, research on the \emph{loss landscape} of neural networks could be useful to understand Bayesian deep learning methods better.
Another future direction is to analyze the forms of Jacobians that different architectures and activation functions give rise to.
This could help to better understand the GLM or GP kernel that we obtain during approximate inference.
Lastly, it is important to understand the further diagonal or factorized approximation applied to the GGN or posterior covariance approximation.
The impact of this further approximation is topic of many studies but has not been fully solved yet.

The transformation from probabilistic neural networks to simpler linear and Gaussian process models is potentially useful for many applications.
For linear models, many interesting quantities can be computed efficiently and in a numerically stable way.
One particularly exciting avenue is to use the marginal likelihood of the underlying linear model to tune neural network hyperparameters.
The marginal likelihood is traditionally one of the arguments for a Bayesian approach but its application to model selection has been rarely used in Bayesian deep learning.
Here, we have shown that this method works well on toy problems but it is important to scale it to more data points and more parameters.
This is also where the Gaussian process variant might be useful:
for a reasonably sized dataset $(\sim 10^4)$, we can conduct a function-space posterior approximation for an arbitrarily large neural network.
This function space approximation leads to the same posterior predictive as a full posterior covariance in the parametric space.
I believe that this could also enable interesting applications in \emph{transfer learning}.
Lastly, the GLM sampling and exact BLR posterior predictives for Bayesian neural networks introduced in this work seem to provide robust uncertainty estimates and, across methods, work better than traditional NN sampling.
It is important to investigate this behavior on other larger datasets.
If this observation is consistent, it could enable better performance on applications that require good uncertainty estimates like \emph{active learning} or \emph{Bayesian optimization} using neural networks.
The closed-form functional form of the posterior due to the Gaussian process connection could also be useful for regularizing neural networks in the function space as opposed to the parameter space.

\cleardoublepage
\phantomsection
\addcontentsline{toc}{chapter}{Bibliography}
\bibliography{thesis.bib}

\appendix
\chapter{Proof of GLM Log Likelihood Derivatives}%
\label{cha:proof_of_glm_log_likelihood_derivatives}

\GLMS*

\begin{proof}
We will present a short proof for continuous distributions with scalar natural parameter and label of the form in \autoref{eq:glm_likelihood_id}.
The applied steps directly extend to discrete distributions and become more tedious for multi-dimensional distributions but are analogous.
We being with the first derivative with respect to the natural parameter:
since $p(y \given f)$ is a probability density that integrates to $1$, we can rewrite \autoref{eq:glm_likelihood_id} as $A(f)=\log \int h(y) \exp \crl{y^\top f} d y$. We start with the first derivative:
\begin{align*}
  \frac{\partial \log p(y \given f)}{\partial f}
  &= y - \frac{\partial A \rnd{f}}{\partial f}
  = y - \frac{\partial}{\partial f}  \log \int h(y) \exp \crl{y f} d y \\
  &= y - \frac{\frac{\partial}{\partial f} \int h(y) \exp \crl{y f} d y}{\exp{A(f)}}
  = y - \frac{\int \frac{\partial}{\partial f} h(y) \exp \crl{y f} d y}{\exp{A(f)}} \\
  &= y - \frac{\int y h(y) \exp \crl{y f} d y}{\exp{A(f)}}
  = y - \myexpect \sqr{Y}.
\end{align*}
We used the dominated convergence Theorem to exchange integral and differentiation in the second line.
For the second derivative, we have
\begin{align*}
  \frac{\partial^2 \log p(y \given f)}{\partial f^2}
  &= \frac{\partial}{\partial f} \rnd{y - \myexpect \sqr{Y}}
  =  - \frac{\partial}{\partial f} \int y h(y) \exp{\crl{y f - A(f)}} dy \\
  &= - \int \frac{\partial}{\partial f} y h(y) \exp{\crl{y f - A(f)}} dy
  = - \int y h(y) \exp{\crl{y f - A(f)}} \rnd{y - \frac{\partial A(f)}{\partial f}} dy \\
  &= \rnd{\int y h(y) \exp{\crl{y f - A(f)}} dy}^2 - \int y^2 h(y) \exp{\crl{y f - A(f)}} dy  \\
  &= \myexpect \sqr{Y}^2 - \myexpect \sqr{Y^2}
  = - \myvar \sqr{Y}.
\end{align*}
\end{proof}

\chapter{Completing the Square}%
\label{cha:complete_square}

In the proofs, one of the main techniques used is to simply compute the square which leads to a simplification or useful relation.
Here, we quickly elaborate on what that means.
We assume a symmetric matrix $\mA \in \R^{D\times D}$ that is positive semi-definite, i.e., we can use a pseudo-inverse if necessary.
Further, we have vectors $\vx, \vb \in \R^D$.
Then, we have the following identity:
\begin{equation}
  \label{eq:complete_square}
  \frac12 \vx^\top \mA \vx + \vx^\top \vb = \frac12 \rnd{ \vx + \mA^{-1}\vb }^\top \mA \rnd{ \vx + \mA^{-1}\vb } - \frac12 \vb^\top \mA^{-1} \vb.
\end{equation}
This is exactly the form it shows up in the proofs.
We can always use the pseudo-inverse here.
Let $\mA^+$ be the pseudo-inverse of $\mA$.
Then, this holds because of the two properties (1) $\mA \mA^{+} \mA = \mA$ and (2) $\mA^+ \mA \mA^+=\mA^+$.

\chapter{Natural Gradient Variational Inference}%
\label{ch:ngvi}

For this section, we will make use of some equivalences in NGVI.
To do so, we need two tightly connected parameterizations of the Gaussian posterior approximation:
the natural and expectation parameterization.
We denote the first and second natural parameter by $\natparam, \natparas$ and the mean parameters $\exparam, \exparas$, respectively:
\begin{align}
  \natparam = \mSigma^{-1} \vmu \quad &\textrm{and} \quad \natparas = - \frac12 \mSigma^{-1}, \label{eq:gauss_nat_app}\\
  \exparam = \vmu \quad &\textrm{and} \quad \exparas = \vmu \vmu^\top + \mSigma. \label{eq:gauss_mean_param}
\end{align}
The ELBO can equivalently be rewritten in terms of these parameters.

In contrast to vanilla gradient descent, natural gradient descent uses the underlying information geometry as opposed to the euclidean geometry to update parameters.
That means, we adjust the parameters of our Gaussian posterior approximation in the space of distributions and not in the euclidean space of the parameters of our Gaussian.
Here, we will derive natural variational inference in the natural parameter space.
Let $\veta$ the natural and $\vphi$ the expectation parameters and $\mF(\veta) = \Cov_q \sqr{ \nabla_\veta \log q (\param)}$ the Fisher information matrix of our approximating distribution $q$~\cite{hensman2012fast, zhang2018noisy}.
Then, natural gradient variational inference in natural parameter space with step size $\gamma$ is governed by the following dynamics:
\begin{equation}
  \label{eq:ngvi_update}
  \veta_{t+1} = \veta_{t} + \gamma \mF(\veta)^{-1} \nabla_\veta \elbo(\veta) = \veta_t + \gamma \nabla_\vphi \elbo(\vphi).
\end{equation}
Since the natural gradient in one parameterization provides the gradient in the other~\cite{hensman2012fast}, we have the second equality.
This allows us to derive the updates for our models in natural parameter space without computing the Fisher information.
However, computing gradients with respect to expectation and natural parameters can be inconvenient (especially for backpropagation).
Therefore, it is useful to further use the chain-rule and express the expectation parameter gradients in terms of gradients wrt. $\vmu, \mSigma$.
We have $\vmu = \exparam$ and $\mSigma = \exparas - \sqr{\exparam}^2$.
Therefore, we can simply write using the chain-rule:
\begin{align}
  \nabla_{\exparam} \elbo(\exparam, \exparas) &= \nabla_\vmu \elbo(\vmu, \mSigma) - 2\nabla_\mSigma \elbo(\vmu, \mSigma) \exparam \label{eq:exparam_grad} \\
  \nabla_{\exparas} \elbo(\exparam, \exparas) &= \nabla_\mSigma \elbo(\vmu, \mSigma). \label{eq:exparas_grad}
\end{align}
We are therefore left with the two gradients with respect to the original parameters to obtain the final update.
Recall the form of the ELBO in \autoref{eq:gva_elbo} and use the closed-form derivatives of the KL-divergence of the prior $p$ from $q$:
\begin{align}
  \nabla_\vmu \elbo(\vmu, \mSigma) &= \nabla_\vmu \myexpect_q \sqr{ \log p(\data \given \param) } + \mSigma_0^{-1} \vmu_0 - \mSigma_0^{-1}\vmu  \\
    \nabla_\mSigma \elbo(\vmu, \mSigma) &= \nabla_\mSigma \myexpect_q \sqr{ \log p(\data \given \param) } +\frac12 \mSigma^{-1} - \frac12 \mSigma_0^{-1}.
\end{align}
The gradients with respect to the expected log-likelihood can be rewritten using Bonnet's and Price's Theorems as shown in \autoref{sub:gva}.
However, especially for the gradient with respect to $\mSigma$, different further approximations are possible and yield different algorithms.
Therefore, the next sections will deal with the remaining term and analyze the update in detail.

We obtain the final NGVI in natural parameter space updates by plugging \autoref{eq:exparam_grad} and \autoref{eq:exparas_grad} into \autoref{eq:ngvi_update}.
We take the gradient at iterate $\vmu_t, \mSigma_t$ and write the natural parameter updates:
\begin{align}
\begin{split}
  \mSigma_{t+1}^{-1} \vmu_{t+1} &= \mSigma_t^{-1} \vmu_t + \gamma \sqr{ \mSigma_0^{-1}\vmu_0 - \mSigma_t^{-1} \vmu_t + \nabla_\vmu \myexpect \sqr{ \log p(\data \given \param) } - 2 \nabla_\mSigma \myexpect \sqr{ \log p(\data \given \param) } \vmu_t } \\
  &= (1 - \gamma)\mSigma_t^{-1} \vmu_t + \gamma \mSigma_0^{-1}\vmu_0 + \gamma \sqr{  \nabla_\vmu \myexpect \sqr{ \log p(\data \given \param) } - 2 \nabla_\mSigma \myexpect \sqr{ \log p(\data \given \param) } \vmu_t  } \label{eq:first_nat_update_app}
\end{split}\\
\begin{split}
  - \frac12 \mSigma_{t+1}^{-1} &= - \frac12 \mSigma_{t}^{-1} + \gamma \sqr{ \frac12 \mSigma_t^{-1} - \frac12 \mSigma_0^{-1} + \nabla_\mSigma \myexpect \sqr{ \log p(\data \given \param) }  } \\
  &= (1 - \gamma) \sqr{ - \frac12 \mSigma_{t}^{-1} } + \gamma \sqr{ - \frac12 \mSigma_{0}^{-1}  } + \gamma \nabla_\mSigma \myexpect \sqr{ \log p(\data \given \param) }. \label{eq:secnd_nat_update_app}
\end{split} \end{align}
This form makes apparent that we \emph{combine} the current posterior approximation with the prior usually with a convex combination ($\gamma \leq 1$) of their natural parameters.
The data dependency comes solely via first and second derivative of the expected log likelihood under the approximating distribution.
Another way of writing above update is therefore to identify natural parameters of $p$, $q$ with parameters at iteration $t$, and those arising from gradients of the expected log likelihood~\cite{khan2017conjugate}.
We denote by $q_t(\param)$ the posterior approximation with parameters at iteration $t$.
Then, we can write with another natural parameter $\widetilde{\veta}$ and sufficient statistics $T(\param)$
\begin{equation}
\label{eq:dist_update}
  q_{t+1}(\param) \propto q_{t}(\param)^{(1-\gamma)} p(\param)^\gamma e^{\gamma T(\param)^\top \widetilde{\veta}}.
\end{equation}
Clearly, the natural parameter $\widetilde{\veta}$ is given by the gradient terms in above updates.
We can easily write the product of $q_t$ and $p$ as another Gaussian and in fact this will be our intermediary prior.
The last exponential term will vary depending on how we compute the gradients.

\end{document}